\documentclass{article}

\usepackage{arxiv}
\usepackage[utf8]{inputenc} 
\usepackage[T1]{fontenc}    
\usepackage{url}            
\usepackage{booktabs}       
\usepackage{amsfonts}       
\usepackage{nicefrac}       
\usepackage{microtype}      
\usepackage{lipsum}
\usepackage{graphicx}
\usepackage{amsmath,amssymb,amsfonts}
\usepackage{algorithmic}
\usepackage{graphicx}
\usepackage{textcomp}
\usepackage{xcolor}
\usepackage{subfigure}
\usepackage[linesnumbered,ruled]{algorithm2e}
\usepackage{filecontents}
\usepackage{array}
\usepackage{wrapfig}
\usepackage{multirow}
\usepackage{tabu}
\usepackage{color,soul}
\usepackage{amsthm}
\usepackage{balance}
\usepackage{times}
\usepackage[hidelinks]{hyperref}
\usepackage[small]{caption}
\newtheorem{theorem}{Theorem}
\newtheorem{myLemma}{Lemma}
\newtheorem{myProblem}{Problem}

\newtheorem{property}{Property}
\title{Weighted Aggregating Stochastic Gradient Descent for Parallel Deep Learning}

\author{
 Pengzhan Guo \\
  Department of Applied Mathematics and Statistics\\
  Stony Brook University\\
  Stony Brook, USA, 11794 \\
  \texttt{pengzhan.guo@stonybrook.edu} \\
   \And
  Zeyang Ye \\
  Department of Applied Mathematics and Statistics\\
  Stony Brook University\\
  Stony Brook, USA, 11794 \\
  \texttt{zeyang.ye@stonybrook.edu} \\
  \And
  Keli Xiao \\
  College of Business\\
  Stony Brook University\\
  Stony Brook, USA, 11794 \\
  \texttt{keli.xiao@stonybrook.edu} \\
  \And
  Wei Zhu \\
  Department of Applied Mathematics and Statistics\\
  Stony Brook University\\
  Stony Brook, USA, 11794 \\
  \texttt{wei.zhu@stonybrook.edu} \\
}

\begin{document}
\maketitle
\begin{abstract}
This paper investigates the stochastic optimization problem with a focus on developing scalable parallel algorithms for deep learning tasks. 
Our solution involves a reformation of the objective function for stochastic optimization in neural network models, along with a novel parallel strategy, coined weighted aggregating stochastic gradient descent (\textit{WASGD}). 
Following a theoretical analysis on the characteristics of the new objective function, \textit{WASGD} introduces a decentralized weighted aggregating scheme based on the performance of local workers.
Without any center variable, the new method automatically assesses the
importance of local workers and accepts them according to their contributions. Furthermore, we have developed an enhanced version of
the method, \textit{WASGD+}, by (1) considering a designed sample order and (2) applying a more advanced weight evaluating function. 
To validate the new method, we benchmark our schemes against several popular algorithms including the state-of-the-art techniques
(e.g., elastic averaging SGD) in training deep neural networks for classification tasks.
Comprehensive experiments have been conducted on four classic datasets, including the \textit{CIFAR-100}, \textit{CIFAR-10}, \textit{Fashion-MNIST} and \textit{MNIST}. 
The subsequent results suggest the superiority of the \textit{WASGD} scheme in accelerating the training of deep architecture. 
Better still, the enhanced version, \textit{WASGD+}, has been shown to be a significant improvement over its basic version.

\end{abstract}


\maketitle

\section{Introduction}
Deep learning, largely based on multi-layer neural networks resembling the way human brain functions, and proven to be superior at pattern recognition tasks \cite{zhang2019deep} based on big and unstructured data such as voice and face recognitions, is increasingly touted as the mathematics engine of artificial intelligence.
Along with the rapid growth in data size and problem complexity, the design of 
efficient parallel computing algorithms for deep learning has become
increasingly critical for its optimal real time deployment. 
Stochastic gradient descent (SGD) proposed by Robbins and Monro \cite{robbins1951stochastic} replacing the gradient search from based on the entire data to a random subset of the data, is an efficient optimization technique for big data problems. 
Parallel SGD aims to achieve further speedup of the optimization process while maintaining good learning performance.

There are three main challenges in parallel SGD.
First, for application purpose, people needs more efficient parallel algorithm
for faster real time deployment in big data. 
Secondly, the communication burden is still very heavy (e.g., wasting time for waiting straggling workers). 
Lastly, the proposed parallel method should be stable in most cases to enable convergence. 
Zinkevich \textit{et al.} \cite{zinkevich2010parallelized} proposed the first parallel SGD method called \textit{SimuParallel SGD}, a synchronous algorithm that averages the sum of all local workers' parameters in the training process.
For reducing the waiting time of local workers, Chen \textit{et al.} \cite{chen2016revisiting} added backup workers in their method.
Their method partially solves the idle time problem for synchronous algorithm.
Moreover, Recht \textit{et al.} \cite{recht2011hogwild} opens the access of global parameters and allow workers update global parameters as soon as possible.
Collectively, these recent progress still needs further improvement. 
In terms of high efficiency, the methods proposed in \cite{chen2016revisiting} directly ignores some workers' results and thus limited their efficiencies.
Regarding less communication time, the method proposed in \cite{zinkevich2010parallelized} still fall short due to its time being linearly related to the sample size.
In terms of extensive applications, methods from \cite{zinkevich2010parallelized} and \cite{recht2011hogwild} need special requirements (e.g. convex situation,uniform memory access) that are not always attainable.

To address these issues, we have proposed a novel distributed parallel method coined the Weighted Aggregating SGD (\textit{WASGD}) in a preliminary
work of this paper 
\cite{guo2019weighted}. 
Our method enables local workers to accept the aggregated parameters of all workers with evaluation weights.  
Our main idea is as follows.
In the parallel system, we first let each local worker updates its local parameters. Then, the communication among local workers is based on the local performance on loss. 
Finally, the update of local parameters is related to a weighted combination of the parameters of all local workers taken in time and space. 

To sum up, the main contributions of \cite{guo2019weighted} are four-fold.
First, by assigning weights for different workers based on their performances, \textit{WASGD}, a new parallel SGD method, is proposed. 
A series of loss estimation techniques to balance the performance evaluation accuracy and computation efforts have also been developed.
Second, the convergence of \textit{WASGD} is theoretically proven. 
Third, the estimation method in \textit{WASGD} incurs no extra computing time, which contributes to the better speedup of our method. 
Last, \textit{WASGD} is applied to the convolutional neural networks (CNN). The results consistently confirm the superiority of \textit{WASGD}, indicating its strength in being effective in different applications.

However, there are still some issues with the basic \textit{WASGD} method that needs additional considerations.
As suggested by \cite{lopes2017facial}, different orders may lead to different impacts on the final result.
The basic \textit{WASGD} in \cite{guo2019weighted} goes through the samples in a completely random order and fails to consider the effect of sample order.
Also, its original weight function may result in ineffectiveness in handling different weighting strategies.
Moreover, previous work lacks theoretical analysis on the model variance of weighted cases and sufficient discussions on some important parameters, such as the parameter ($\beta$) that controls the changing percentage of local workers.
Here, we focus on addressing these issues. 
This paper differs from \cite{guo2019weighted} and contributes to the literature further in four ways:
\begin{itemize}
     \item We develop an enhanced version of \textit{WASGD}, namely \textit{WASGD+}, by considering the sample order. 
     The enhanced method also introduces a more flexible weight evaluating method, that is capable of computing weights under different strategies.
    \item 
    By formulating a general form of the parallel optimization
    problem for deep learning, we have theoretically enabled
    the stability of our method. We have also discovered the essential
    connection between parallel SGD and mini-batch
    by changing the communication period, which guides
    us in regulating the performance of parallel SGD more
    efficiently.
     \item We perform comprehensive parameter analysis based on the new method.
     We numerically analyze the impact of the parameter $\beta$ and find that the full acceptance of the communication result is not always the optimal choice. 
     We also compare the performances under different weighting strategies and find the condition when weighted case is better than equally weighted case.
    \item  To show the consistency and practical value of the new method, in addition to conducting original experiments, we conduct further test a more complex dataset, \textit{CIFAR-100}.
    We also conduct numerous additional experiments for the selection of weighting strategy and the value of $\beta$.
    All results suggest the superority of \textit{WASGD+} and its stability in handling different deep learning applications.
\end{itemize}
The rest of the paper is organized as follows. In Section~\ref{sec:preliminary}, we introduce the overall backgrounds, including the preliminary definitions and the problem formalization. Section~\ref{sec:methodology} presents the update rule for \textit{WASGD+} based on the modified objective function, along with our weight estimation, weight evaluating function and sample order generation methods.
Following that, we provide the convergence and variance analysis for \textit{WASGD+}  under our experiment settings in Section~\ref{sec:conv} and Section~\ref{sec:varana}.
In Section~\ref{sec:ex}, we show numerous experiment results on multiple datasets and detailed analysis on \textit{CIFAT-10} and \textit{CIFAR-100}.
After that, we summarize additional related work in Section~\ref{sec:relatedwork} and finally conclude in
Section~\ref{sec:conclusion}.
\section{Preliminaries and Problem Statement}
\label{sec:preliminary}
We consider minimizing the function $F(x)$ in a parallel environment    \cite{bertsekas1989parallel} with \textit{p} workers. 
The stochastic optimization problem can be formulated as follows.
\begin{equation}
   \min_{x} F(x):= \min_{x} E\left[f(x,\xi)\right],
\end{equation}
where $x$ represents model parameters and $\xi$ is a random variable that follows a probability distribution $Q$ over $\Omega$ such that $F(x)=\int_\Omega f(x,\xi)Q(\xi)d\xi$. 
In the parallel environment, as discussed in \cite{ye2018unified}, one needs a nuanced formulation of the objective function. 
We found a good starting point in the Elastical Averaging SGD (\textit{EASGD}) method \cite{zhang2015deep} where it reformulated the objective function with $p$ workers as:
\begin{equation}\label{eqn:2}
  \min_{x^1,x^2...x^p,\tilde{x}}\sum_{i=1}^{p}\left(E\left[f(x^i,\xi^i)\right]+\frac{\lambda}{2}\left\|x^i-\tilde{x}\right\|^2\right),
\end{equation}
where $x^i$ is the variable (parameters) of local worker $i$; $\tilde{x}$ represents the parameters for the center variable, which is stored by the master processor that manages the communication.
Related update rules are as follows.

\begin{equation}\label{eqn:3}
x^{i}_{t+1}=x^i_t-\eta g^i_t \left( x^i_t \right) - \alpha\left( x^i_t-\tilde{x} \right),
\end{equation}
and
\begin{equation}\label{eqn:4}
\tilde{x}_{t+1}=(1-p\alpha)\tilde{x}+\alpha{\sum_{i=1}^{p}x^i_t},
\end{equation}
where $\eta$ is the learning rate; $\alpha=\eta\lambda$ represents the moving rate.
During the communication, \textit{EASGD} lets local workers only communicate with the center variable, and the moving rate $\alpha$ is preferred to be small leading to more exploration efforts of local workers. 
This setting can result in a serious issue in the searching process.

For the main method in \cite{zhang2015deep}, assuming that the communication
order of local workers denotes their indices and the communication
period is $\tau$. 
After finishing communicating with all the workers, the modified parameters of the master can be written as:
\begin{equation}\label{eqn:upeasgd}
 \tilde{x}_{t+p}=(1-p\alpha)^p \tilde{x}_t+\sum _{i=1}^p \alpha (1-p\alpha)^{p-i}x^i_t. 
\end{equation}
Since $1-p\alpha<1$, Equation \eqref{eqn:upeasgd} indicates that the maximum weight can only be assigned to $\tilde{x}_t$ or the slowest worker $x_t^p$.

If we assume that the maximum weight was given to the master's previous parameters $\tilde{x}_t$, by applying Bernoulli's inequality \cite{bullen2013handbook}, we get $(1-p\alpha)^p \geq 1-p^2\alpha$. Then we have:
\begin{equation}\label{eqn:condeasgd}
    \begin{split}
        &
        (1-p\alpha)^p > \alpha
        \Leftarrow 1-p^2\alpha>\alpha \Leftarrow \alpha < \frac{1}{1+p^2}.
    \end{split}
\end{equation}
Equation \eqref{eqn:condeasgd} provides a necessary condition when the biggest weight was given to master's previous parameters.
Due to the preference of small $\alpha$, the center master $\tilde{x}$ has higher possibility to be allocated the maximum weight.
Given  that $\tilde{x}$ initialized at the beginning of the algorithm and a small $\alpha$, the center variable must have a high chance to keep bad parameters for local workers, which will kill a large amount of time to fix in the algorithm.

To address the unreasonable assignment issue, we propose to remove $\tilde{x}$ and introduce a new term which we call \textit{weighted aggregation} in the quadratic penalty term \cite{nocedal2006numerical} to generate a more healthy allocation, which shares the similar communication strategy in \cite{ye2018multi, ye2018applying}. 
The combination term aggregates the parameters of all local workers and assigns larger weights to better performed local workers.
Such settings can make sure all workers search around the current optimization, and enable local workers to do many explorations at the same time. 
Thus, we propose a new format of the stochastic optimization as follows.
\begin{myProblem}[The Parallel Stochastic Optimization Problem]

\begin{equation}\label{eqn:5}
  \min_{x^1,x^2...x^p}\sum_{i=1}^{p}\left(E\left[f(x^i,\xi^i)\right]+\frac{\lambda}{2}\left\|x^i-\sum_{j=1}^{p}\frac{w^{j}x^{j}}{\sum_{k=1}^{p}w^k}\right\|^2\right),
\end{equation}
where ${w^j}$ denotes the weight for $j^{th}$ worker; ${x^i}$ represents the parameters of local worker $i$; and $\xi^i$ follows the same distribution $Q$.
\end{myProblem}

Note that we consider all processors can sample the entire dataset. 
The problem of the equivalence of our objective and the original ones is studied in the literature and is known as the \textit{augmentability} or the \textit{global variable consensus} problem \cite{boyd2011distributed,hestenes1975optimization}. 
The penalty term $\lambda$ is expected to control all local workers search near the current optimal parameters. 

\section{Methodology}
\label{sec:methodology}
This section introduces the update rule of \textit{WASGD} and then discuss its enhanced version \textit{WASGD+}.
 \begin{figure}[t]
\centering
\includegraphics[width=8cm,height=5cm]{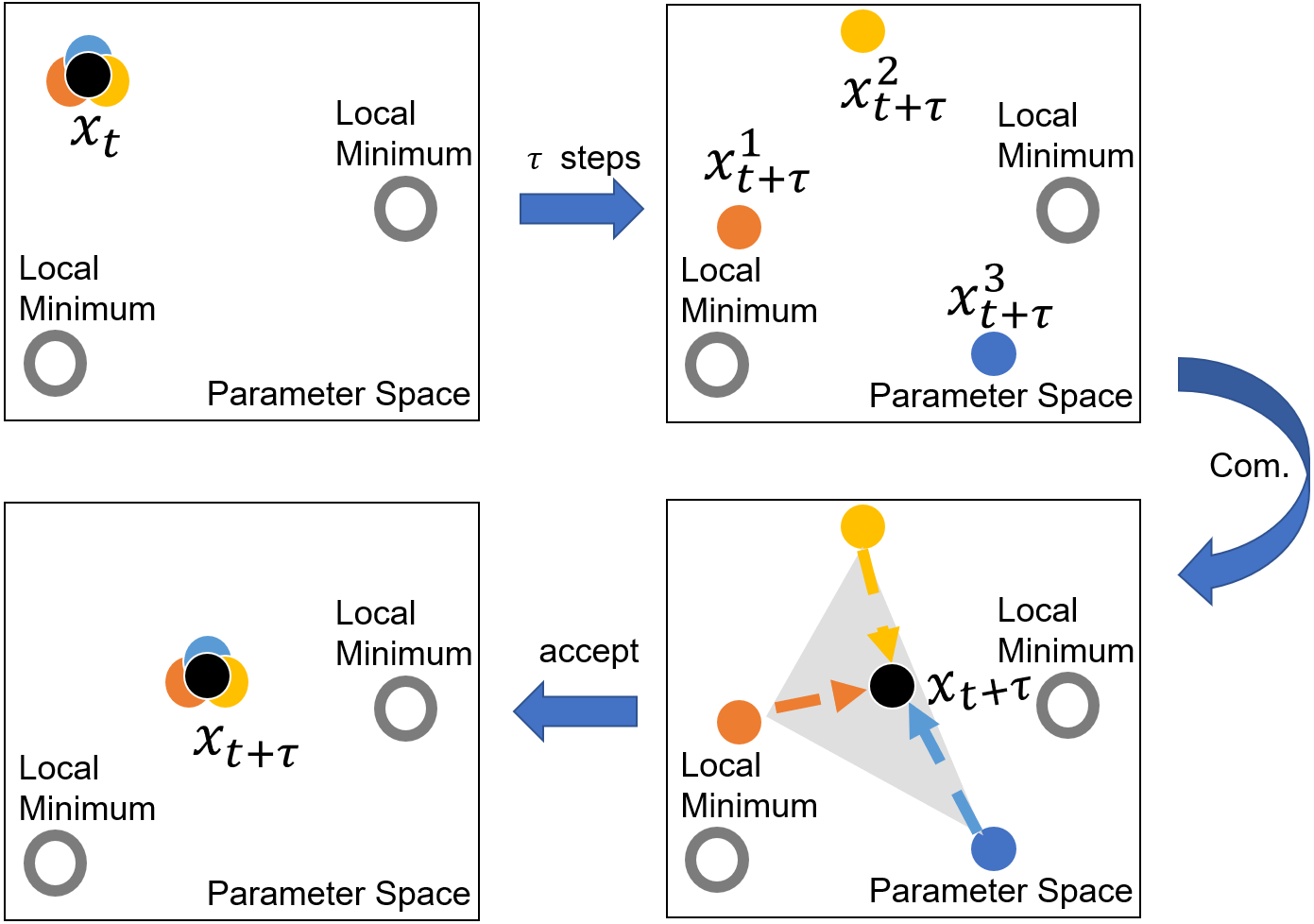}
\caption{One Communication Period for \textit{WASGD+} when $\beta \neq 1$}\label{fig:communication}
\end{figure}
\subsection{Update Rule for WASGD}

Given $w^j$ the weight of the local variable $x^j$, let $\theta^j=\frac{w^j}{\sum_{i=1}^{p}w^i}$, Equation \eqref{eqn:5} can be rewritten as:

\begin{equation} \label{eqn:updateRule}
\min_{x^1,...x^p}\sum_{i=1}^{p}\left(E[f(x^i,\xi^i)]+\frac{\lambda}{2}\left\|x^i-{\sum_{j=1}^{p}\theta^j x^{j}}\right\|^2\right).
\end{equation}
If we take the stochastic gradient descent with respect to $x^i$ , then we get:
\begin{equation}\label{eqn:6}
x^{i}_{t+1}=x^i_t-\eta g^i_t \left( x^i_t \right) - \eta\lambda \left( x^i_t-{\sum_{j=1}^{p}\theta^j x^{j}_t} \right),
\end{equation}
where $\eta$ is the learning rate , if we denote $\beta=\lambda\eta$, then Equation \eqref{eqn:6} becomes to:
\begin{equation}\label{eqn:7}
x^{i}_{t+1}=x^i_t(1-\beta) + \beta  {\sum_{j=1}^{p}\theta^j x^{j}_t}-\eta g^i_t \left( x^i_t \right),
\end{equation}
where $\beta \in [0,1]$. $\beta=0$ indicates a full rejection of the aggregation results, so there is no communication between each local worker.
$\beta=1$ indicates the system fully accepts the aggregation results.
If $\beta \in (0,1)$, each local worker negotiates itself with the combination result. 
The bigger $\beta$, the more compromise of the local worker. 
Since deep neural networks usually have numerous local minima \cite{choromanska2015loss,reddi2016stochastic}, which have similar values of loss, a suitable $\beta$ can help us improve the convergence rate.
We will show the selection of $\beta$ with more details in \autoref{sec:ex}.


As shown in \autoref{fig:communication}, suppose we have three local workers and only need to reach one of the local minima, as any of them would lead to a good enough loss function value.
First, two local workers are attracted by the local minimum on the right, and one local worker is attracted by the left local minimum.
After $\tau$ iterations, the three local workers stop and communicate with each other.
Then, since the blue worker and the yellow worker both move toward the local minimum on the right and have good performance, the orange worker leaves the left local minimum after the communication.
They will restart at new positions closer to the local minimum on the right.

Based on Equation \eqref{eqn:7}, the basic \textit{WASGD} was proposed in the preliminary work, and
detailed analysis can be found in \cite{guo2019weighted}. 
The main algorithm or \textit{WASGD} can be found in the Appendix.
Here, we focus on discussing the enhanced version of the method, \textit{WASGD+}, which will introduce improvements on the weight evaluating function and the sample order strategy.
\begin{algorithm}[t]
\SetAlgoLined
\SetKwInOut{Input}{Input}
 \Input{learning rate $\eta$, decision choice $\beta$, estimation number $m$, divided order part number $n$, divided communication part number $c$, communication period $\tau \in N$, M samples in datasets $D$, initial parameters $x^i=x$, $p$ local workers.}
 
 B $\leftarrow$ RecordIndex($D, m,c,\tau$)\;
 Scores, Seed $\leftarrow$ zeros($n, 1$)\;
 \While{The stopping criteria is not met}{
 \For{$l$ = 0,1,...,$n-1$}{
 $h^i,score \leftarrow 0$\;
 $A,$ Seed[$l$]=OrderGen(Scores[$l$], Seed[$l$], $\frac{M}{n}$);\\
 \For {k =$1,...,\frac{M}{n}$}{
  $x\leftarrow x^i$\;
 \If{ $k \in B$}{
        $h^i \leftarrow h^i+loss(x^i,D\left[l\frac{M}{n}+A[k]\right]);$\\
       
        }
  
  \If{ \textit{k} divides $\tau$ }
  {
    Send $(h^i,x_k^i,i)$ to $p-1$ workers\;
    Wait for$(h^j,x_k^j,j)$ from $p-1$ workers\;
    Arrange $h$ and $x$ based on received index\;
    $\theta^i \leftarrow \frac{e^{-\tilde{a}h^i/\sum_{j=1}^p h^j}}{\sum_{g=1}^p e^{-\tilde{a}h^g/\sum_{j=1}^p h^j}}$\;
    $x^i \leftarrow x^i(1-\beta)+\beta{\sum_{j=1}^{p}\theta^j x^{j}}$\;
    $score+\leftarrow$ Judge($h^1,...h^p,i,p$)\;
    $h^i \leftarrow 0$ \;
  }
    
    $x^i \leftarrow x^i-\eta g^i_{ite}(x,D\left[l\frac{M}{n}+A[k]\right])$\;
      }
   Scores[$l$]=$score$
 }
 }
 \caption{Synchronous WASGD+: Processing by worker \textit{i}}
 \label{alg:1}
\end{algorithm}
\subsection{Weight Evaluating function}
Based on \cite{bottou1991stochastic} and \cite{rere2015simulated}, we learn that SGD shares many properties with Simulated Annealing (SA).
The core of the SA is controlled by the Boltzmann distribution that determines the possibility of accepting updates and guiding the result.
Inspired by the fact that the Boltzmann distribution is very successful in SA,
we also apply it to our communication process.

During the communication, the most important part is how to
weight the results; naturally, we assume that the weight should be
inversely proportional to the loss energy, which in turn means that if the loss
of the $i^{th}$ worker is small, then we should strengthen the weight of
this worker, and vice versa.
Since different tasks have different requirements for weight \cite{yeh1999task} , the optimal weight function, subsequently, can not
always be the same \cite{zhang2019iterative}.
The weight evaluating function for \textit{WASGD} is $\frac{1}{h}$ which can hardly meet such optimality requirement.
We propose a new weight evaluating function based on the Boltzmann distribution \cite{wolf20005}:
\begin{equation}
    w=e^{-\tilde{a}h} (\tilde{a} \in \mathbf{R}),
\end{equation}
where $h$ is the loss energy.
If $h \rightarrow 0$, then $e^{-\tilde{a}h} \rightarrow 1$, the weight of each sample is equivalent to $\frac{1}{p}$.
In order to avoid this, we normalize the energy before applying.
The weight function should be defined as: 
\begin{equation} \label{eqn:weight}
w^i = e^{-\tilde{a}h'^i},
\end{equation}
where $h'^i=\frac{h^i}{\sum_{j=1}^p h^j}$.
Then the normalized weight of $i^{th}$ worker should be:
\begin{equation} \label{eqn:weight1}
    \theta^i = \frac{e^{-\tilde{a}h'^i}}{\sum_{k=1}^p e^{-\tilde{a}h'^k}},
\end{equation}
where $p$ is the number of local workers.
Equation \eqref{eqn:weight1} shares the same form as Boltzmann distribution and we can have the following property:
\begin{property}\label{property1}
If $\tilde{a} \rightarrow \infty$, the weighting strategy is equivalent to broadcast the best performance case  while $\tilde{a} \rightarrow 0$, the weighting strategy is equivalent to equally weighted case.  
\end{property}
\begin{proof}
Given $p$ workers, then the weight for $i^{th}$ worker is:
\begin{equation} \label{eqn:pro1}
    \begin{split}
        & \theta^i = \frac{e^{-\tilde{a}h'^i}}{\sum_{k=1}^p e^{-\tilde{a}h'^k}}\\
        & = \frac{1}{1+\sum_{k=1,k \neq i}^p e^{-\tilde{a}(h'^k-h'^i)}}.
    \end{split}
\end{equation}
Assuming that $\tilde{a} \rightarrow \infty$ and $i^{th}$ worker has the best performance, then for any other worker $j$, we know that $h'^i-h'^j<0$.
Based on this inequality, we can conclude that:
\begin{equation} \label{eqn:pro2}
    e^{-\tilde{a}(h'^j-h'^i)} \rightarrow \infty.
\end{equation}
From Equation \eqref{eqn:pro2}, we know that at least on term of denominator in Equation \eqref{eqn:pro1} is approaching $\infty$.
Thus, we know the weight of $k^{th}$ work is:
\begin{equation}
    \theta ^k = \frac{1}{\infty}\approx 0 (k \neq i).
\end{equation}
For the $i^{th}$ worker which has the smallest loss energy, since $h^k-h^i\geq0 (k \in p)$, we can conclude that:
\begin{equation} \label{eqn:pro3}
    \begin{cases}
    e^{-\tilde{a}(h'^k-h'^i)} \rightarrow 0 (k \neq i)\\
    \theta ^i=\frac{1}{1+(p-1)\cdot 0} \approx 1
    \end{cases},
\end{equation}
which means that the best performance's local worker will dominate the communication result as mentioned in \cite{zhang2019parallel}.

If $\tilde{a} \rightarrow 0$, then for $i^{th}$ worker, $e^{-\tilde{a}(h'^k-h'^i)}=1$.
Based on Equation \eqref{eqn:pro1}, we can find that:
\begin{equation}
    \begin{split}
        \theta^i & = \frac{1}{1+\sum_{k=1,k \neq i}^p 1}
         = \frac{1}{p},
    \end{split}
\end{equation}
which is equivalent to equally weighted case.
\end{proof}
Property \autoref{property1} shows that the new weight evaluating function is more flexible and can satisfied most cases as required.
\begin{property}\label{property2}
If we set equally weighted case as baseline and try different values of $\tilde{a}$. 
As $\tilde{a} \rightarrow \infty$, the performance of weighted case must be worse than the baseline,  while as $\tilde{a} \rightarrow 0$, the weighted case will approach the performance of the baseline. 
\end{property}
The proof of Property \autoref{property2} is provided in the Appendix.
Property \autoref{property2} provides the tendency of extreme case in terms of $\tilde{a}$ and we can find the proper value of $\tilde{a}$ more efficiently based on such tendency.
As \cite{chu1999parallel} had successfully applied the Boltzmann distribution in the mixing of states of parallel SA, we also apply it as a weight evaluating function in the parallel SGD case.
\begin{algorithm}[h!]
\SetAlgoLined
\textbf{Function 1}: RecordIndex($D, m,c,\tau$):\\
B$\leftarrow \emptyset;$\\
\For{i = 0,1,...c}{
\For{j=0,1,...$\frac{m}{c}$}{
B.append($(i+1) \frac{\tau}{c} - j - 1$)\;
}
}
Output:B\;

\textbf{Function 2}: OrderGen(total-score, old-seed, $l$):\\
\eIf{total-score $\leq$ $-1$}
{
seed $\leftarrow$ old-seed\;
}{
seed $\leftarrow$ generate a new random seed\;
}
Generate a random sequence $A$ of length $l$ with $seed$; \\
Output: A, seed;

\textbf{Function 3}:Judge($h^1,...h^p,i,p$):\\
$ave\leftarrow \frac{\sum_{j=1}^p h^j}{p}$\;
$stdv \leftarrow \sqrt{\frac{1}{p-1}\sum_{j=1}^p (h^j-ave)^2};$\\
$score \leftarrow \frac{h^i-ave}{stdv};$\\
Output: score;
 \caption{Main Functions}
 \label{alg:2}
\end{algorithm}
\begin{figure*}[t]
\centering
\subfigure[]{\includegraphics[width=0.22\linewidth]{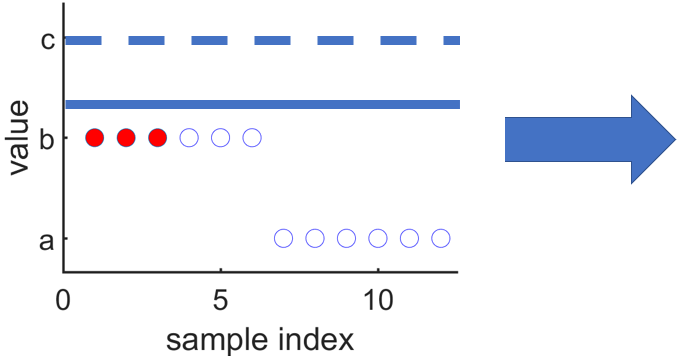}}
\subfigure[]{\includegraphics[width=0.22\linewidth]{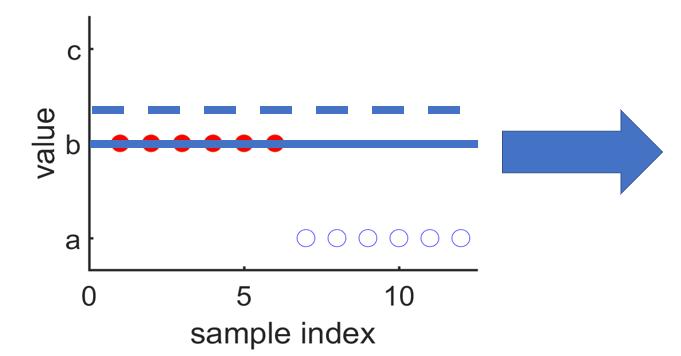}}
\subfigure[]{\includegraphics[width=0.22\linewidth]{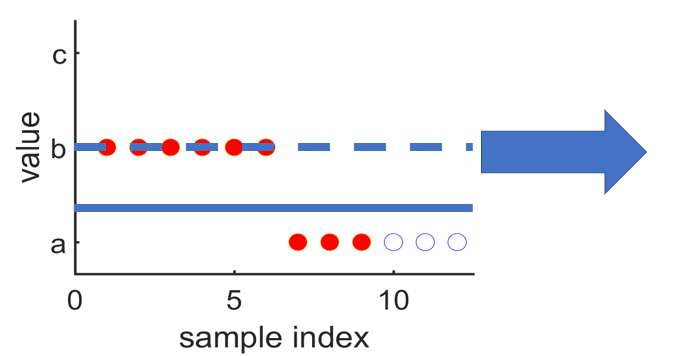}}
\subfigure[]{\includegraphics[width=0.22\linewidth]{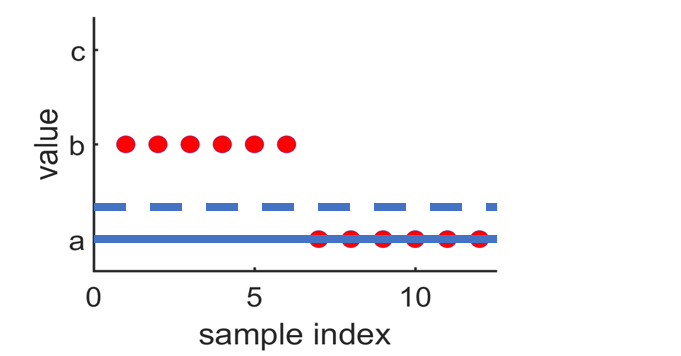}}\\
\subfigure[]{\includegraphics[width=0.22\linewidth]{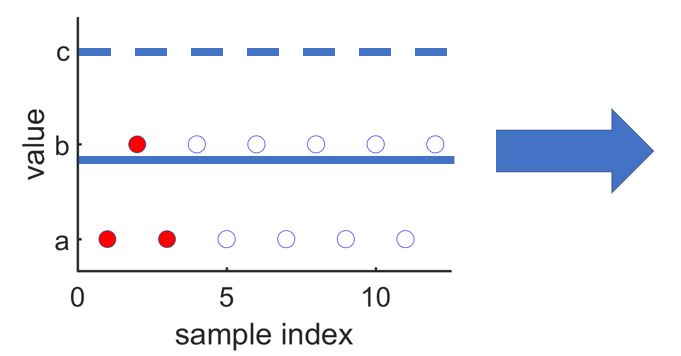}}
\subfigure[]{\includegraphics[width=0.22\linewidth]{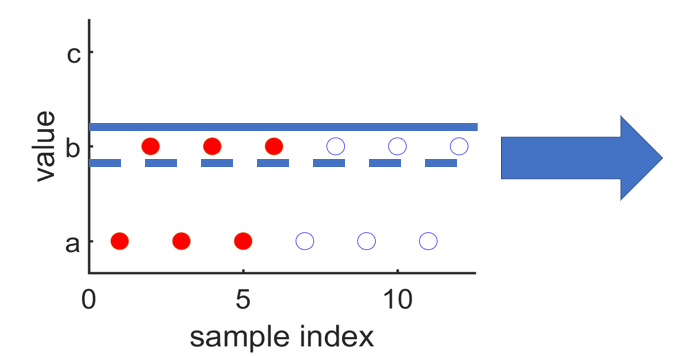}}
\subfigure[]{\includegraphics[width=0.22\linewidth]{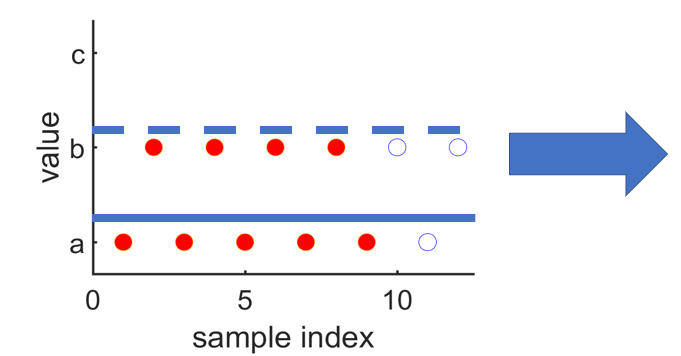}}
\subfigure[]{\includegraphics[width=0.22\linewidth]{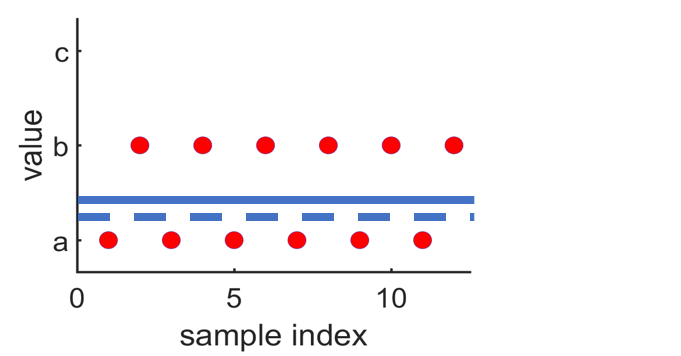}}
\caption{Results of different sample orders.}\label{fig:sampleorder}
\end{figure*}
\subsection{Weight Accuracy Estimation}
As stated above, let $F$ denote the objective function, the normalized value
of the loss energy during communication about the $i^{th}$ worker should be
calculated as:
\begin{equation}
h'^i_{\tau} = \frac{\sum_{j=1}^{N}F(x_{\tau-1}^i,\xi^{j})}{\sum_{k=1}^p \sum_{j=1}^{N}F(x_{\tau-1}^k,\xi^{j})},
\end{equation}
where $\tau$ denotes the communication period, $p$ is the number of workers, and $N$ is the size of training samples. 
Based on Equation \eqref{eqn:weight}, the weight of $i^{th}$ worker is defined as::
\begin{equation} \label{eqn:gpz1}
\theta^{i}_{true} =\frac{e^{-\tilde{a}\frac{\sum_{j=1}^{N}F(x_{\tau-1}^i,\xi^{j})}{\sum_{k=1}^p \sum_{j=1}^{N}F(x_{\tau-1}^k,\xi^{j})}}}{\sum_{l=1}^p e^{-\tilde{a}\frac{\sum_{j=1}^{N}F(x_{\tau-1}^l,\xi^{j})}{\sum_{k=1}^p \sum_{j=1}^{N}F(x_{\tau-1}^k,\xi^{j})}}},
\end{equation}
This step means we need to go through the whole training dataset in order to get the weight.
The time for computing this process is depending on the size of dataset.
If the sample size is very small, then it will not take too much time.
For example, in \textit{CIFAR-10}, the sample size is equal to half of a hundred thousand, computing these samples is very expensive. 

We can use $m$ samples to estimate the weight, then Equation \eqref{eqn:gpz1} becomes:
\begin{equation} \label{eqn:gpz2}
\theta^{i}_{true} \approx \theta^{i}  
=\frac{e^{-\tilde{a}\frac{\sum_{j=1}^{m}F(x_{\tau-1}^i,\xi^{d_j})}{\sum_{k=1}^p \sum_{j=1}^{m}F(x_{\tau-1}^k,\xi^{d_j})}}}{\sum_{l=1}^p e^{-\tilde{a}\frac{\sum_{j=1}^{m}F(x_{\tau-1}^l,\xi^{d_j})}{\sum_{k=1}^p \sum_{j=1}^{m}F(x_{\tau-1}^k,\xi^{d_j})}}}
(m \leq N),
\end{equation}
where $d_j$ denotes the index of samples.
Equation \eqref{eqn:gpz2} will spend nearly $\frac{m}{N}$ time of Equation \eqref{eqn:gpz1}.
However, it still needs extra computing time.
The loss function for classification has been defined as:
\begin{equation}
F^i(x)=-log \frac{e^{s_j}}{\sum_k e^{s_k}},   
\end{equation}
where $s_k \in z(x^i,\xi^i)$, $z$ is the given function, $j$  denotes the index of label, $k \in [0, n_1]$, and $n_1$ is the number of label's type. 
During the first step of back propagation, we can get
\begin{equation}
s_j= z_j= \sum_{k=1}^{N'} \left(W_{jk}z'_k({x^i,\xi^i})+b_{jk}\right),    
\end{equation}
where $N'$ is the size of hidden layers, $z'_k$ denotes function of previous result.   
If we want to update the value of $W_{jk}$, based on the chain rule, we will have
\begin{equation}\label{eqn:chainrule}
\frac{\partial F_i}{\partial W_{jk}}=\frac{\partial F_i}{\partial z_j}\frac{\partial z_j}{\partial W_{jk}}.   
\end{equation}

From Equation \eqref{eqn:chainrule}, we know that calculating $\frac{\partial F_i}{\partial z_j}$ needs the value of $\frac{e^{s_j}}{\sum_k e^{s_k}}$, which means the current loss energy of $i^{th}$ sample does not need extra computing time.
If we can use the loss energy during the back propagation, then the weight evaluation process does not need extra computing time, we reformulate Equation \eqref{eqn:gpz2} as : 
\begin{equation} \label{eqn:gpz3}
\begin{split}
  &\theta^i 
  \approx \frac{e^{-\tilde{a}\frac{\sum_{j=1}^{m}F(x_{\tau-j}^i,\xi^{d_j})}{\sum_{k=1}^p \sum_{j=1}^{m}F(x_{\tau-j}^k,\xi^{d_j})}}}{\sum_{l=1}^p e^{-\tilde{a}\frac{\sum_{j=1}^{m}F(x_{\tau-j}^l,\xi^{d_j})}{\sum_{k=1}^p \sum_{j=1}^{m}F(x_{\tau-j}^k,\xi^{d_j})}}}.
\end{split}
\end{equation}
Based on Equation \eqref{eqn:gpz3}, the loss should be recorded at the same time and space for all workers, this will restrict the exploration for local workers.
Allowing more explorations can lead to the improved performance \cite{zhang2015deep}.
It is unfair if we record the loss at the same space due to the fact that different workers will go through different training steps.
So we only record the loss at the same time, in order to avoid outlier, following \cite{xie2020joint}, we assign the record data with a novel assignment distribution before communication.
\begin{equation} \label{eqn:gpz4}
\begin{split}
  &\theta^i \approx \frac{e^{-\tilde{a}\frac{\sum_{j=1}^{m}F(x_{e_j}^i,\xi^{d_j})}{\sum_{k=1}^p \sum_{j=1}^{m}F(x_{e_j}^k,\xi^{d_j})}}}{\sum_{l=1}^p e^{-\tilde{a}\frac{\sum_{j=1}^{m}F(x_{e_j}^l,\xi^{d_j})}{\sum_{k=1}^p \sum_{j=1}^{m}F(x_{e_j}^k,\xi^{d_j})}}},
\end{split}
\end{equation}
where $e^j$ denotes the index of last $\frac{m}{c}$ steps of each $\frac{\tau}{c}$ iterations.
The error for Equation \eqref{eqn:gpz4} has been defined as:
\begin{equation}\label{err1}
    error=\sum_{i=1}^p |\theta^i-\theta^i_{true}|
\end{equation}
The range for the error is between zero to two.
We can also apply this estimation method on \textit{Multiplicative Weight Update Method}\cite{dwork2014algorithmic} which also needs the weight to update the possibility of selection.
\subsection{Sample Order}
The correct weighting of samples can speed up the convergence rate  \cite{needell2014stochastic}.
However, this process needs to find the right distribution of
the samples causing extra analysis time.
The right weight of the samples can be reflected by the best order of the samples that an SGD goes through.

In SGD, each sample will guide the solution to a direction (gradient).
The effect of samples in the same directions
of the current sample will be reinforced and the effect of samples
in the opposite directions of the current sample will be
diminished.
Therefore, the order of the samples for SGD affects
the solution quality in one epoch, which in turn will influence the
overall convergence rate.
This is illustrated by the example below.

As shown in \autoref{fig:sampleorder}, we are given 12 samples where half of them have the value $b$ and the other half have the value $a$. 
We want to determine the value of $d$ in the function $y=d$ such that it has the smallest least square error among these samples.
To emphasize the importance of the order of the sample, we assign each sample an index.
This index denotes the order of going through these 12 samples when we run SGD.
The solid red points indicate used samples and the blue empty points refer to the unused samples.
The dashed blue line represents the previous solution while the solid
blue line corresponds to the current solution obtained by the SGD
steps based on the solid samples.
\autoref{fig:sampleorder} (a), (b), (c) and (d) show the fitting procedure with one sample order.
The initial solution starts at $y=c$, in (a) and (b) the solutions are modified by six SGD steps based on the six samples with the value $b$.
The solution will arrive at a position which will infinitely approach $y = b$ from the above.
Then SGD improves the current solution by going through the rest six samples with the value $a$.
Finally, the solution will be close to $y=a$.
This is an inappropriate solution as we know that the optimal solution is $y=(a+b)/2$.
\autoref{fig:sampleorder} (c), (d), (e) and (f) show the procedure with another sample order where the samples with value $a$ are separated by the samples with value $b$.
The current solution will be optimized for both $a$ and $b$ for every two SGD steps and will be closer to $y=(a+b)/2$.

When we are training among large dataset, it has high possibility to select some samples with the same label without grouping the samples in advance.
The time of sample grouping is proportional to its size.
Even we group the samples, for a dataset with $N$ samples, the number of sample orders is $(N!-c)$, where $c$ denotes a few known bad orders.
Such process is equivalent to select some data in a dataset which can maximally differentiate these data from others, which is proved NP-hard \cite{liu2012differentiating}.
It is nearly impossible to try all these orders and find the best sample order even in a parallel environment.
However, in parallelization, we can still obtain a relatively better sample orders by comparisons.

In each communication procedure, we assign a score to each
local worker based on its performance.
To determine the score, we first find the average performance, which is the mean of the losses from all the workers at the current step.
Furthermore, we calculate the fluctuation, that is, the standard deviation of these losses.
Then, we normalize each local worker’s current loss by subtracting
the average performance and dividing by the fluctuation and defining it
as the score.
Assume that the losses for all the workers at the same step follow a normal distribution.
Then, the score follows a standard normal distribution with mean zero and standard deviation one.
By empirical rule, if this score is less than minus one, its performance is better than $84\%$ of the local workers.
We consider the sample order of this local worker relatively better and save it for the next epoch.


For some bad orders, they have some wayward parts during the
searching process.
The bad performance is caused by the other parts' unsuitable order.
In order to generate relatively better sample orders, we can retain these special parts and shuffle the other parts.
As suggested by many previous studies (e.g., \cite{jin2013shared}), learning from different views in the training process should enhance the model performance.
Thus, different orders of workers guarantees that the aggregation result will be improved under different positions of local workers.
This method has a higher possibility to generate a better order than randomly shuffling all the samples.
\subsection{WASGD+}
Following \cite{watcharapichat2016ako}, we introduce a distributed deep learning system without a parameter server.
We use $e^{-\tilde{a}h'^i}$ to evaluate the weight of the $i^{th}$ worker.
The detailed procedure is demonstrated in Algorithm \ref{alg:1}.
Each local worker is only responsible for updating its own parameters and going through the samples in different orders. 
Note that the iteration will increase by one after each update, and the local workers will wait when the iterations are the multiples of communication period $\tau$. 
When iterations is divisible by $\tau$, the worker will hold the search until it receives an aggregation result based on all workers. 
During communication, the estimated
loss will be used to generate score and weight for the performance of the current parameters.
After fixed iterations, each local worker will compare its score with the judgment. 
If the score satisfied the judgment, the current order will be preserved, and otherwise replaced by a new one.

The choice of asynchronous or synchronous algorithm of our method is based on the communication period and the time difference for computing each sample. 
If the time difference for computing each sample is large and the communication period is small, the total computing time for local worker has higher possibility to vary a lot from each other. 
Then we choose asynchronous algorithm. 
If the time difference in computing each sample is large and communication period is large, the possibility that computing time varies substantially from each other has been decreased, then we use the synchronous algorithm. 
If the time difference is small which means each local worker completes the work almost at the same time, then we use the synchronous algorithm.

Since the time difference for computing samples in \textit{MNIST}, \textit{Fashion-MNIST}, \textit{CIFAR-10} and \textit{CIFAR-100} is small, we use the synchronous version.
The algorithm will be terminated when the required loss is reached.
We also provide the asynchronous version of our method in the Appendix.
\section{Theoretical Analysis}
In this section, we will discuss the convergence and stability of our method.
\subsection{Convergence Analysis}
\label{sec:conv}
Recall that Equation \eqref{eqn:updateRule} is the update rule for the parameter vector $x$ and we set $\beta$ to be one.
Then the solutions from all the workers of \textit{WASGD} or \textit{WASGD+} are the same in every $\tau$ steps.
We reformulate the update rule as:
\begin{equation} \label{eqn:update1}
x_{t+\tau} = \sum_{i=1}^{p} \left[ \theta^i_t  \sum_{j=t}^{\tau-1}  \left[ x^{i}_j -  \eta g^i_j \left( x^i_j \right) \right] \right].
\end{equation}
As shown above, the convergence of \textit{WASGD+} is equivalent to the convergence of the sequence of the solutions $\{x_{l\tau}\}^\infty_{l=1}$.
Let $S_{t, k}^i=(s_{t, l}^i)^k_{l=1}$ be a sequence of sample index at time step $t$. We can define $h_{t, k}^i$ and $H_{t, k}^i$ for the $i^{th}$ worker as:

\begin{equation} \label{eqn:update2}
h_{t, k}^i(x|s_{t, k}^i)=
x - \eta g_{s^i_{t, k}} (x)
\end{equation}
and
\begin{equation} \label{eqn:update3}
H_{t, k}^i(x|S_{t, k}^i)=
\begin{cases}
h_{t,1}^i(x|s_1^i), & k = 1\\
h_{t, k}^i(\cdot|s_{t, k}^i) \circ H_{t, k-1}^i(x|S_{t, k-1}^i), & k > 1
\end{cases}
\end{equation}
Then, Equation \eqref{eqn:update1} can be rewritten as:
\begin{equation} \label{eqn:update4}
x_{t+\tau}= \phi (x_t)= \sum_{i=1}^{p} \left[ \theta^i_t \cdot H_{t, \tau}^i(x_t|S_{t, \tau}^i) \right].
\end{equation}

Assume that $F(x)$ is convex and Lipschitz continuous.
We first introduce the lemma below: 
\begin{myLemma}\label{lemma:1}
$\forall i, t, k$, when the learning rate in $h_{t, k}^i$ is low, let $(P(R, m), W_z)$ be the space of distributions, where $(R, m)$ is a Rondon space, $P(R, m)$ is the set of all distributions over $(R,m)$, and $W_p$ is the Wasserstein distance between two distributions. If we define $\varphi_{t, k}^i$ by applying $h_{t, k}^i$ pointwise to R, $\varphi_{t, k}^i$ is a contraction mapping.
\end{myLemma}
\begin{proof}
This lemma can be obtained in \cite{zinkevich2010parallelized}.
It first shows that $h_{t, k}^i$ is a contraction mapping and then proves that the induced mapping $\varphi_{t, k}^i$ is a contraction mapping over $(P(R, m), W_p)$. 
\end{proof}
By the connection between $h_{t, k}^i$ and $\varphi_{t, k}^i$, we can study $\Phi$.
\cite{zinkevich2010parallelized} discussed the situation that $\Phi_t$ is the linear combination of $\varphi_{t, k}^i$, which in our case is when $\tau=1$.
We extend its proof to $\tau>1$ where the composite functions exist.
\begin{theorem}\label{thm:1}
Let $\Phi$ be the mapping in $(P(R, m), W_p)$ induced by $\phi$ in Equation \eqref{eqn:update4} by applying $\phi$ pointwise to R. Then $\Phi$ is a contraction mapping.
\end{theorem}

\begin{proof}
$\forall R_1, R_2 \in R$, by the definition of $\Phi$, we have

\begin{equation} \label{eqn:thm_1}
\begin{split}
    & W_p(\Phi(R_1), \Phi(R_2))  \\ 
    & = \sum_{i=1}^{p} \left[ \theta^i_t \cdot \varphi_{t, \tau}^i \circ \varphi_{t, \tau-1}^i \circ ... \circ \varphi_{t, 1}^i (W_p(R_1, R_2)) \right].
\end{split}
\end{equation}
Since $\varphi_{t, k}^i$ is a contraction mapping when $1 \leq k \leq \tau$, there exists a Lipschitz constant $c_{t, k}^i < 1$ such that 
\begin{equation} \label{eqn:thm_2}
\begin{split}
    & \text{RHS of Equation \eqref{eqn:thm_1}}   \\ 
    & \leq \sum_{i=1}^{p} \left[ \theta^i_t \cdot c_{t, 1}^i \cdot \varphi_{t, \tau}^i \circ ... \circ \varphi_{t, 2}^i (W_p(R_1, R_2)) \right].
\end{split}
\end{equation}
By induction, we have:
\begin{equation}
\begin{split}
    & \text{RHS of Equation \eqref{eqn:thm_2}}   \\ 
    & \leq \sum_{i=1}^{p} \left[ \theta^i_t \cdot \prod_{k=1}^{\tau} c_{t, k}^i  (W_p(R_1, R_2)) \right] \\
    & \leq \sum_{i=1}^{p} \left[ \theta^i_t \cdot \prod_{k=1}^{\tau} c_{t, k}^i   \right] \cdot W_p(R_1, R_2).
\end{split}
\end{equation}
Since $c_{t, k}^i < 1$ and $\sum_{i=1}^{p} \theta^i_t = 1$, $\Phi$ is a contraction mapping in $(P(R, m), W_p)$.
\end{proof}
By Theorem \ref{thm:1}, due to the properties of contraction mapping, the sequence $\{x_{l\tau}\}^\infty_{l=1}$ will converge in the exponential rate.
\subsection{Variance Analysis}
\label{sec:varana}
Based on \cite{zhang2016parallel}, we propose a general version of variance under different weighted strategy about the following function: $F(x)=\frac{1}{2}cx^2$.
By using gradient samples of the form $g(x)=cx-\tilde{b}x-\tilde{h}$
where $\tilde{b} , \tilde{h}$ are random variables with mean zero and variance $ \sigma_{\tilde{b}}^2, \sigma_{\tilde{h}}^2$.
$p$ local workers will choose to communicate with each other with the probability $\zeta$ after each step.
We can define the variance of our method as shown in Lemma 2 below.
\begin{myLemma}\label{lemma:2}
Given $\omega, \delta$, where $\omega=\sum_{i=1}^p(\theta^i)^2, \delta=\frac{\zeta}{(1-\zeta)\eta(2c-\eta c^2)}$, $\eta$ is learning rate and $\zeta$ is communication probability,
the asymptotic variance of weighted aggregating SGD is
\begin{equation}
    \begin{split}
    & lim_{t \rightarrow \infty} Var\left(\sum_{i=1}^p {\theta^i} x^i_t\right)\\
    & = \eta\sigma_{\tilde{h}}^2 \omega \left(2 c -\eta c^2-\eta\sigma_{\tilde{b}}^2\frac{1+ \delta \omega}{1+\delta}\right)^{-1}.
    \end{split}
\end{equation}
\end{myLemma}
\begin{proof}
The update rule of $i^{th}$ worker is:
\begin{equation}
    x^i_{t+1}=(1-\eta c)x^i_t+\eta (\tilde{b}x^i_t+\tilde{h}).
\end{equation}
Assume we initialize $x^i_0$ such that $E[x^i_0]=0$ which implies $E[x^i_t]=0$.
Due to the independence of each update, we can denote the variance of communication result as:
\begin{equation}
    Q_{t}=E\left[\left(\sum_{i=1}^p \theta^i x^i_t\right)^2\right].
\end{equation}
Then $Q_{t+1}$ can be shown as:
\begin{equation}\label{var1}
\begin{split}
    & Q_{t+1}=E\left[\left(\sum_{i=1}^p \theta^i x^i_{t+1}\right)^2\right]\\
    & = (1-\eta c)^2 Q_{t}+\eta^2  \sigma_{\tilde{b}}^2 \sum_{i=1}^p (\theta^i)^2 E\left[( x^i_t)^2\right] + \eta^2 \sigma_{\tilde{h}}^2 \sum_{i=1}^p(\theta^i)^2.\\
\end{split}
\end{equation} 
For any local worker, we have that :
\begin{equation}
    \begin{split}
        & E[(x_{t+1}^{i})^2] 
        = (1-\eta c)^2 E\left[(x^i_{t})^2\right]+ \eta^2 \sigma_{\tilde{b}}^2 E\left[(x^i_t)^2\right] +  \eta^2 \sigma_{\tilde{h}}^2.
    \end{split}
\end{equation}
Assuming that $E[(x^i_t)^2]$ is independent of $i$, let $P_t=E[(x^i_t)^2]$, we can conclude that:
\begin{equation} \label{Var2}
    \begin{cases}
    P_{t+1}=(1-\eta c)^2 P_t+ \eta^2 \sigma_{\tilde{b}}^2  P_t  +  \eta^2 \sigma_{\tilde{h}}^2\\
    Q_{t+1}=(1-\eta c)^2 Q_{t}+ \eta^2 \sigma_{\tilde{b}}^2 P_t \omega+  \eta^2 \sigma_{\tilde{h}}^2\omega,
    \end{cases},
\end{equation}
where $\omega=\sum_{i=1}^p(\theta^i)^2$.
Equation \eqref{Var2} represents the variance and expectation without communication.
If workers communicate with each other, then the update rule is:
\begin{equation} 
    x^i_{t+1}=\sum_{i=1}^p \theta^i x^i_{t}.
\end{equation}
By the definition, we can get that:
\begin{equation} 
    \begin{split}
        Q_{t+1} & = E\left[(\sum_{i=1}^p \theta^i x^i_{t+1})^2\right]\\
        & = E\left[(x^i_{t+1})^2\right] = P_{t+1}\\
        & = E\left[(\sum_{i=1}^p \theta^i x^i_{t})^2\right] = Q_t.\\
    \end{split}
\end{equation}
With the probability $\zeta$ and law of total expectation, replacing $\sum_{i=1}^p(\theta^i)^2$ by $\omega$, 
in order to find a steady-state, we need to solve the linear equation:
\begin{equation}\label{Var4}
\begin{split}
  & \left[
    \begin{array}{ccc}
    Q  \\
    P  \\
\end{array}
\right]\\
    &= (1-\zeta)
    \left[
        \begin{array}{ccc}
        (1-\eta c)^2Q+\eta^2 \sigma_{\tilde{b}}^2 \omega P+  \eta^2 \sigma_{\tilde{h}}^2\omega \\
        (1-\eta c)^2 P+ \eta^2 \sigma_{\tilde{b}}^2   P +  \eta^2 \sigma_{\tilde{h}}^2 \\
    \end{array}
    \right]
    +
    \zeta
    \left[
        \begin{array}{ccc}
        Q \\
        Q  \\
    \end{array}
    \right].
\end{split}
\end{equation}
Let $\rho=(1-(1-\eta c)^2)$ and $\delta \rho=\frac{\zeta}{1-\zeta}$. 
Since the determinant of the $2 \times 2$ matrix is $\Delta=\rho^2\omega^{-1}(1+\delta)-\rho \eta^2  \sigma_{\tilde{b}}^2 (\omega^{-1}+\delta)$,we can know the value of Q is:
\begin{equation}\label{eqn:finvar}
\begin{split}
    Q & =\left[
    \begin{array}{ccc}
    1 & 0
\end{array}
\right] 
\left[
    \begin{array}{ccc}
    \rho \omega^{-1} &  - \eta^2 \sigma_{\tilde{b}}^2 \\
    -\delta \rho &    (1+\delta)\rho- \eta^2 \sigma_{\tilde{b}}^2\\
\end{array}
\right]^{-1}
\left[
    \begin{array}{ccc}
     \eta^2 \sigma_{\tilde{h}}^2\\
     \eta^2 \sigma_{\tilde{h}}^2 \\
\end{array}
\right]\\
    & =
\left[
    \begin{array}{ccc}
    1 & 0
\end{array}
\right]
\frac{1}{\Delta}
\left[
    \begin{array}{ccc}
    (1+\delta)\rho- \eta^2 \sigma_{\tilde{b}}^2 &   \eta^2 \sigma_{\tilde{b}}^2 \\
    \delta \rho &    \rho \omega^{-1}\\
\end{array}
\right]
\left[
    \begin{array}{ccc}
     \eta^2 \sigma_{\tilde{h}}^2\\
    \eta^2 \sigma_{\tilde{h}}^2 \\
\end{array}
\right]\\
        &=\eta\sigma_{\tilde{h}}^2\omega \left(2 c -\eta c^2-\eta\sigma_{\tilde{b}}^2\frac{1+ \delta \omega}{1+\delta}\right)^{-1}.
    \end{split}
\end{equation}
\end{proof}
The above Lemma shows that the variance of weighted aggregating SGD can be determined and controlled based on Equation \eqref{eqn:finvar}. 
A predetermined variance will result in the stability of our method.
\begin{myLemma}\label{lem3}
On a large data, given p workers, if communication probability $\zeta=1$, equally weighted case is equivalent to mini-batch gradient descent with the same learning rate $\eta$.
\end{myLemma}
\begin{proof}
Given the update process of one local worker at current step:
\begin{equation}
x^{i}_{t+1}=x_t-\eta g^i_t \left( x_t \right),
\end{equation}
where $x_t=\frac{1}{p} \sum_{i=1}^p x^i_{t}$.
Then the communication result can be shown as:
\begin{equation}
    \begin{split}
        x_{t+1} 
        & =\frac{1}{p} \sum_{i=1}^p x^i_{t+1}\\
        &= \frac{1}{p} \sum_{i=1}^p \left(x_t-\eta g^i_t \left( x_t \right)\right)\\
        & = x_t-\frac{\eta}{p} \sum_{i=1}^p  g^i_t(x_t).
    \end{split}
\end{equation}
Since the sample size $N$ is large enough, $N \gg p$.
The probability that $g^i_t(x_t)=g^j_t(x_t)$ is nearly zero.
\end{proof}
As suggested by \cite{shang2018vr}, mini-batching can effectively reduce the variance of stochastic gradient estimates. 
Our Lemma \autoref{lem3} further serves as a boundary condition regarding the performance of parallel stochastic gradient descent.
\begin{figure}[t]
  \centering
  \subfigure[p=8]{\includegraphics[width=0.25\linewidth]{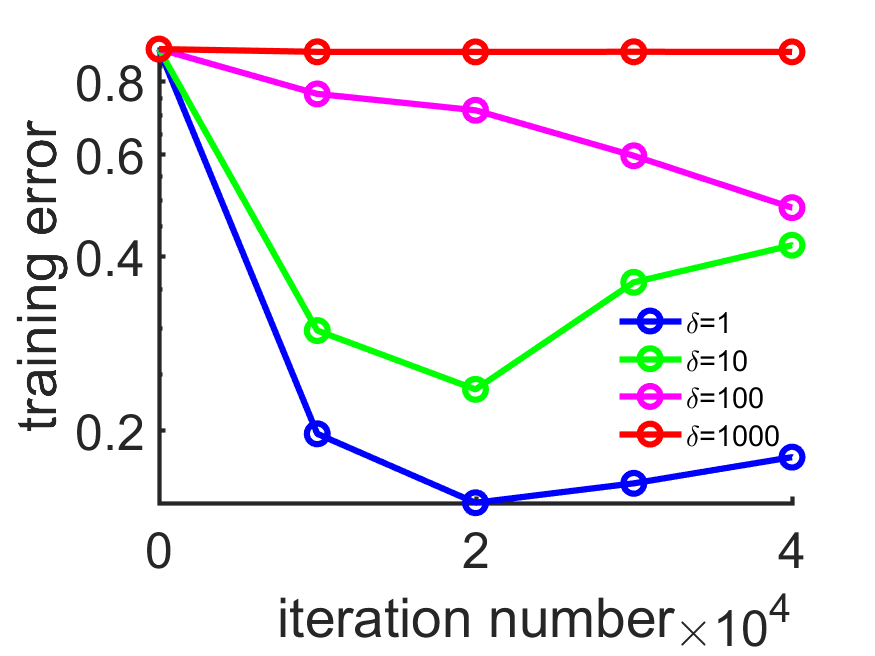}}
  \subfigure[p=8]{\includegraphics[width=0.25\linewidth]{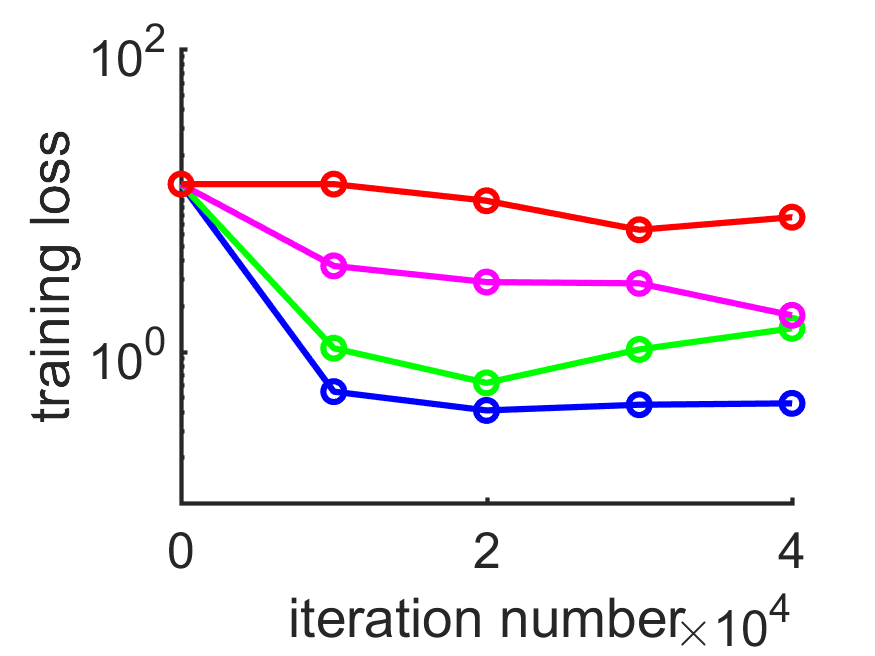}}\\
  \subfigure[p=4]{\includegraphics[width=0.25\linewidth]{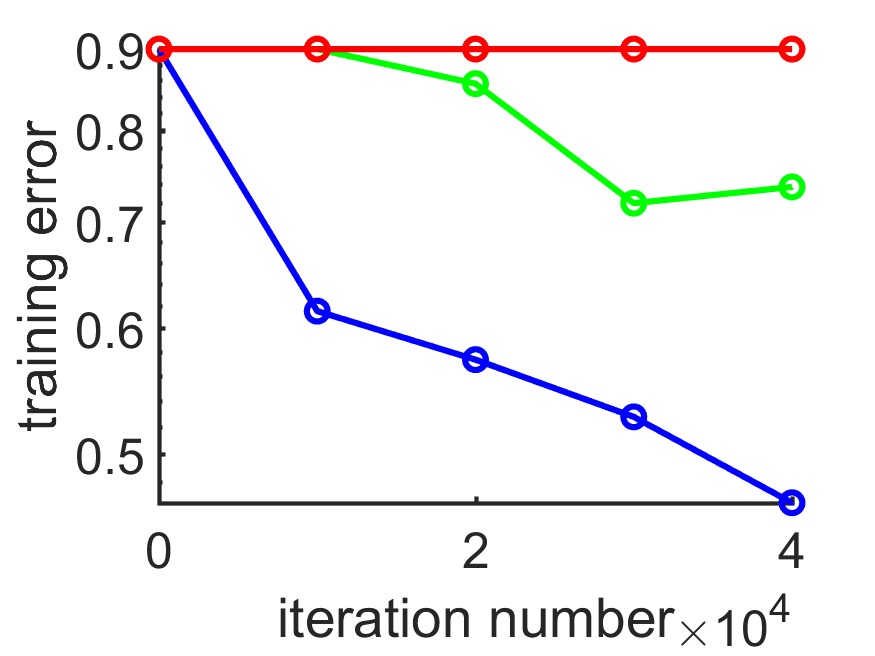}}
  \subfigure[p=4]{\includegraphics[width=0.25\linewidth]{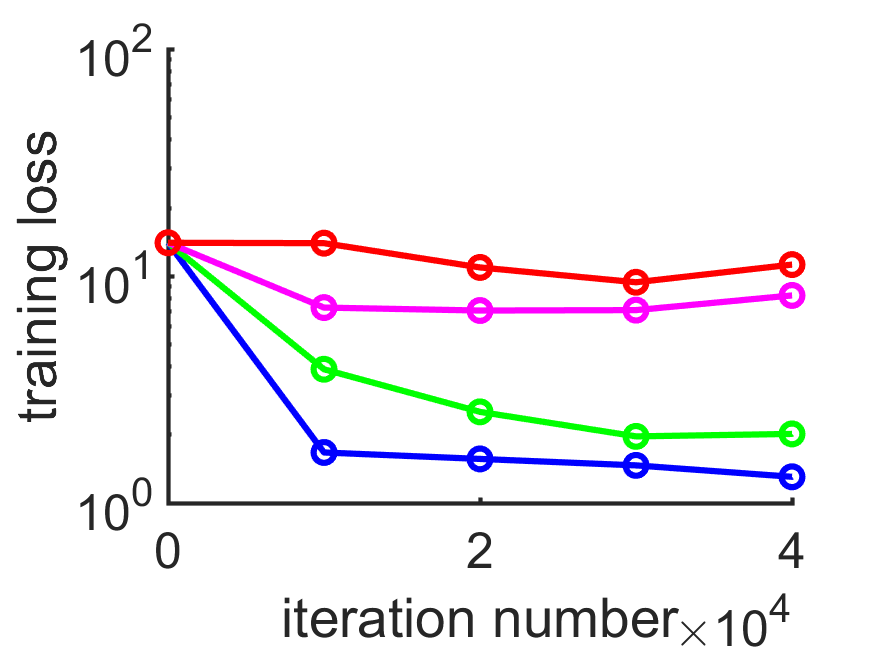}}
  \caption{Order effect. (a) and (b) represent the results on \textit{Fashion-MNIST}, (c) and (d) show the results on \textit{CIFAR-10}.} \label{fig:order effect}
\end{figure}

\section{Experiments}

\label{sec:ex}
We evaluate our method by applying it to the convolutional neural network (CNN) for specific classification tasks.
\subsection{Order effect}
Sample order is a key factor in influencing the aggregation performance.
A proper sample order can help improve the performance of local worker and modify the communication result. 
We numerically prove the effect of sample order by performing 
certain tests \textit{Fashion-MNIST} and \textit{CIFAR-10} datasets with $\delta=1,10,100,1000$ where $\delta$ denotes the numbers of continuously gone-through samples with the same label.

As shown in \autoref{fig:order effect}, if $\delta=1000$ which means that the local worker was given only one type of label during each communication period, the improvement in accuracy and loss can be largely ignored.
For $\delta=100$, although it sometimes achieves better performance compared with $\delta=1000$, the convergence rate is still very slow.
For $\delta=1,10$, their performance is better than $\delta=100,1000$. 
As \textit{CIFAR-10} is more complex than \textit{Fashion-MNIST}, the difference between $\delta=1$ and $\delta=10$ on the training error is becoming quite enlarged.
Based on such phenomenon, we can learn that the more complex the dataset, the more important the sample order.

Although $\delta=1$ achieves the best result among the other three situations, it is still inefficient to conclude that $\delta=1$ is the best sample order among all the orders.
Even we can find the best sample order for one dataset, it might not be suitable for the other datasets. 
Given that it is hard to try all the possible orders, our parallel
method can help us approach the optimal sample order as soon as possible.
\subsection{Experiment settings}
In this section, we will discuss the general experiment settings for the datasets and the detailed settings for different benchmarks.

\subsubsection{General settings}
Now we show the experiments on four datasets\footnote{ MNIST (http://yann.lecun.com/exdb/mnist/), Fashion-MNIST (https://\\arxiv.org/abs/1708.07747), CIFAR-10 and CIFAR-100 (https://www.cs.toronto\\.edu/~kriz/cifar.html)}.
For the \textit{CIFAR-10} and \textit{CIFAR-100} datasets, we implement CNN with eight convolution layers and four fully connected layers.
Following \cite{zhang2015deep}, given \textit{C} the fully-connected convolutional operator, \textit{M} the max pooling operator, \textit{D} the linear operator with dropout rate being one, and \textit{F} the linear operator with softmax output, our CNN structure can be described as (3,32)\textit{C}(64,32)
\textit{M}(64,16) \textit{C}(128,16) \textit{M}(128,8) \textit{C}(256,8) \textit{M}(256,4) \textit{C}(512,4)
\textit{M}(512,2) \textit{D}(128,1) \textit{D}(256,1) \textit{D}(512,1) \textit{D}(1024,1) \textit{F}(10,1).
For the \textit{MNIST} and \textit{Fashion-MNIST} datasets, we implement the following 6-layer CNN:
(1,28) \textit{C}(16,24) \textit{M}(16,12)
\textit{C}(32,8)
\textit{M}(32,4).

We set the constant learning rate for all experiments, $\eta=0.001$ for \textit{CIFAR-10, CIFAR-100} and $\eta=0.01$ for \textit{Fashion-MNIST, MNIST}.
We check the cross entropy loss of the whole training examples every 10,000 iterations. 
The running time includes two parts: computation and communication. 
For \textit{CIFAR-10} and \textit{CIFAR-100}, we test the algorithm performance under different numbers of Tesla K80 ($p=2,4,8$).
For \textit{Fashion-MNIST} and \textit{MNIST}, we implement the experiments on different numbers of CPUs ($p=4,8,16)$.
In order to better present the tendency of performance on error, we take the logarithm of the results.
Due to space limitation, we show the detailed experiments on \textit{CIFAR-10} and \textit{CIFAR-100} and the main results about accuracy on \textit{Fashion-MNIST} and \textit{MNIST}.
To obtain reliable and consistent experimental results, all experiments are conducted under TensorFlow \cite{abadi2016tensorflow}.
\subsubsection{Benchmark settings}
We compared our \textit{WASGD+} with five parallel methods and one sequential method:

\begin{itemize}

\item SGD \cite{ruder2016overview}. The standard sequential SGD with constant learning rate $\eta$.

\item \textit{SimuParallelSGD (\textit{SPSGD})} \cite{zinkevich2010parallelized}. The algorithm divides the data into \textit{p} parts, and each local worker only gets one part. After several iterations they average the sum of all parameters. 

\item \textit{Elastical Averaging SGD (EASGD)} \cite{zhang2015deep}. It is an efficient asynchronous optimization method. Local workers only communicate with the center variable. During the communication, the center variable changes itself based on the parameters of local workers and also updates the parameters of the local workers. We empirically set the communication period $\tau=50$. Based on \cite{zhang2015deep}, we set $\alpha=\frac{0.9}{p}$ for \textit{CIFAR-10,CIFAR-100} and $\alpha=\frac{0.009}{p}$ for \textit{Fashion-MNIST}, \textit{MNIST}. 

\item \textit{Original Multiplicative Weight Update Method (\textit{OMWU})} \cite{dwork2014algorithmic}. This is a classical method published in 1957.
The weights of workers' performance influence their chosen
possibilities in the next $\tau$ steps. As long as the iterations are large enough, it can find the best performance of local workers.
We compute the total loss at the last step before communication to evaluate the weight and the communication period {$\tau=1000$}.

\item \textit{Modified Multiplicative Weight Update Method (\textit{MMWU})}. We modify the process of calculating the weight. 
The weight in the original work was based on the performance of the
entire training dataset. 
We replaced it with the proposed estimation approach and the $\tau=1000$.

\item \textit{Weighted Aggregating SGD (\textit{WASGD})} \cite{guo2019weighted}. This is a synchronous
version of the parallel method that we have published in 2019.
The weights of the workers are determined by the inverse of their loss energies. 
The communication result is aggregating based on the weights of local workers and is totally accepted, so $\beta=1$.
For achieving the same estimation accuracy as \textit{WASGD+}, we set $m=150$, and the communication period {$\tau=1000$}.
\end{itemize}

\subsection{Parameter analysis in \textit{WASGD+}}
Now, we discuss the impact of parameter selections on the performance of the proposed method.

\subsubsection{$\tilde{a}$ analysis}

\begin{figure*}[h]
    \center
  \subfigure[p=8]{\includegraphics[width=0.22\linewidth]{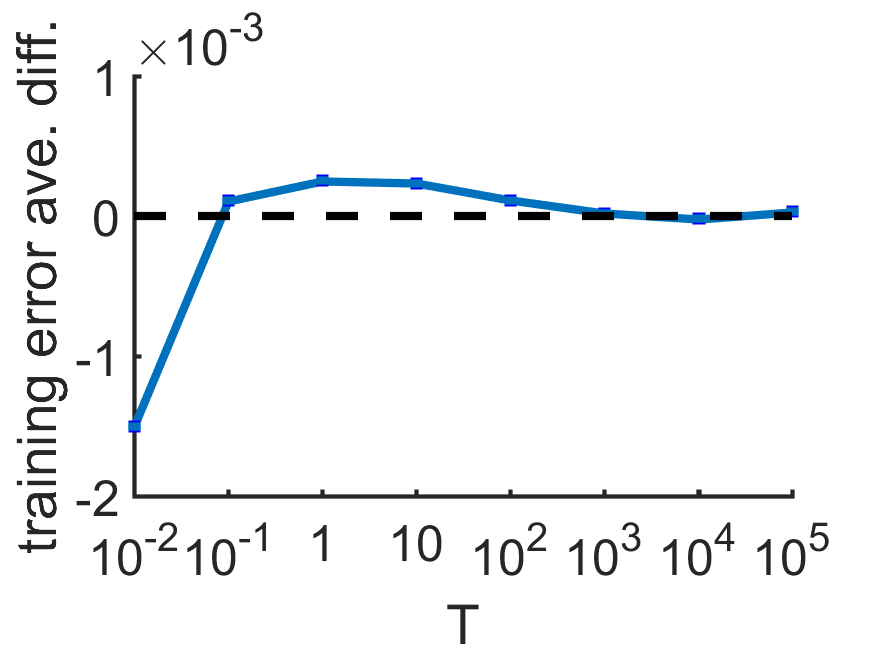}}
   \subfigure[p=8]{\includegraphics[width=0.22\linewidth]{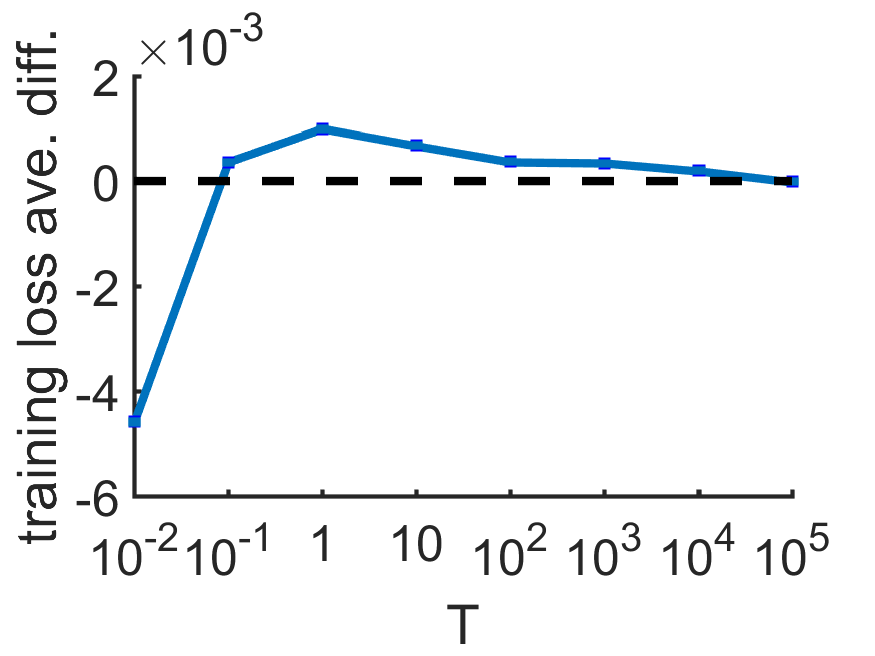}}
  \subfigure[p=8]{\includegraphics[width=0.22\linewidth]{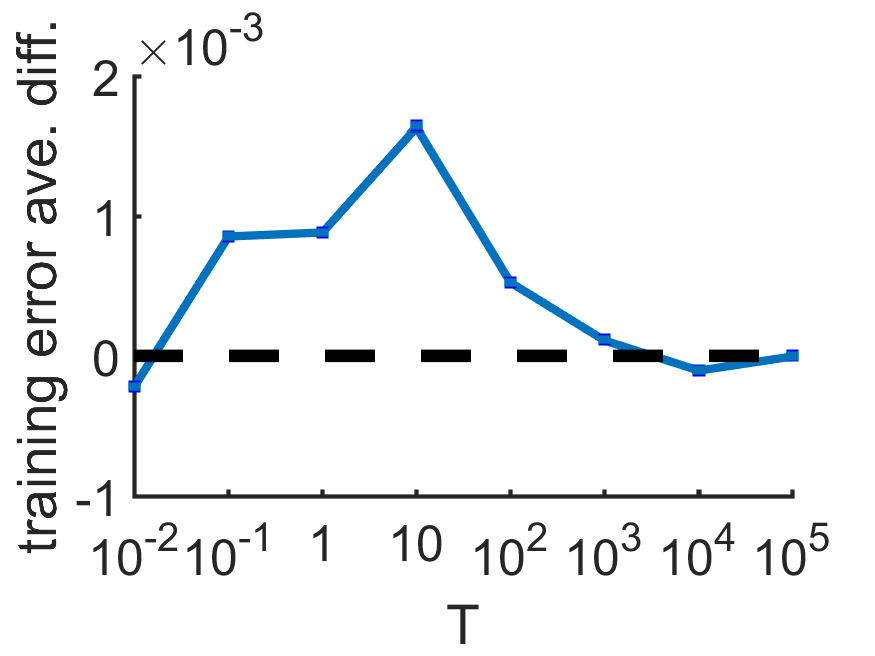}}
  \subfigure[p=8]{\includegraphics[width=0.22\linewidth]{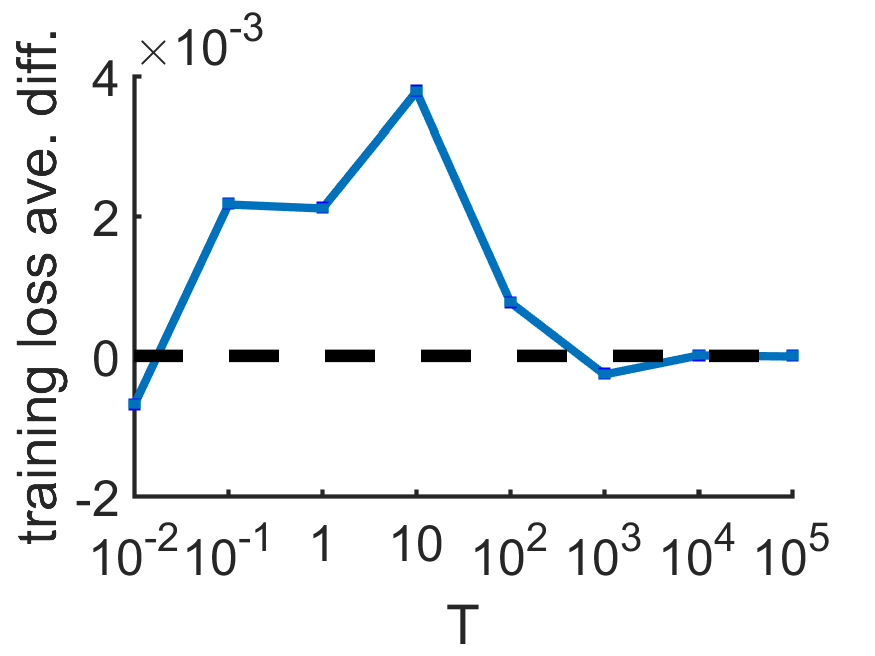}}\\
  \subfigure[p=4]{\includegraphics[width=0.22\linewidth]{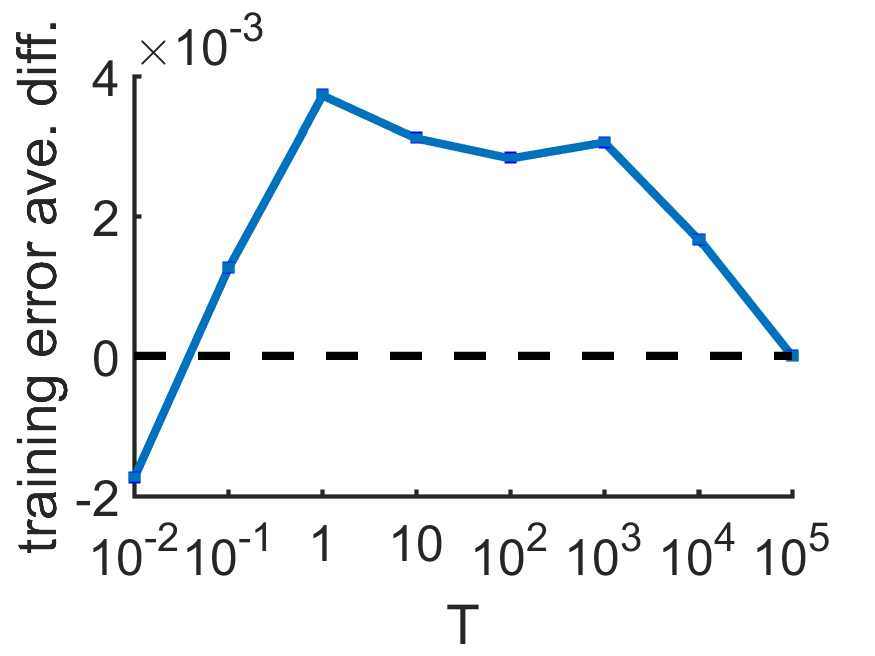}}
  \subfigure[p=4]{\includegraphics[width=0.22\linewidth]{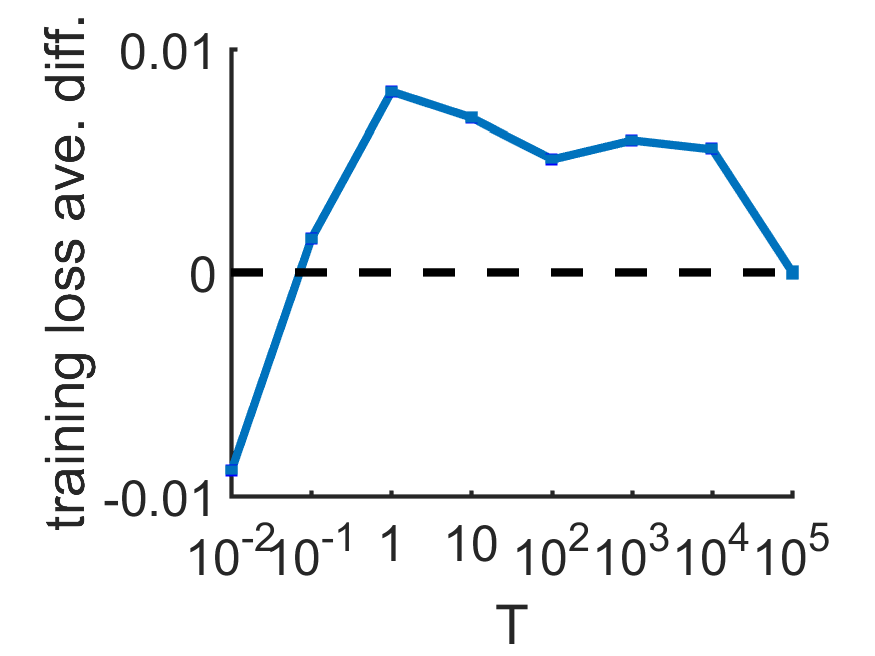}}
  \subfigure[p=4]{\includegraphics[width=0.22\linewidth]{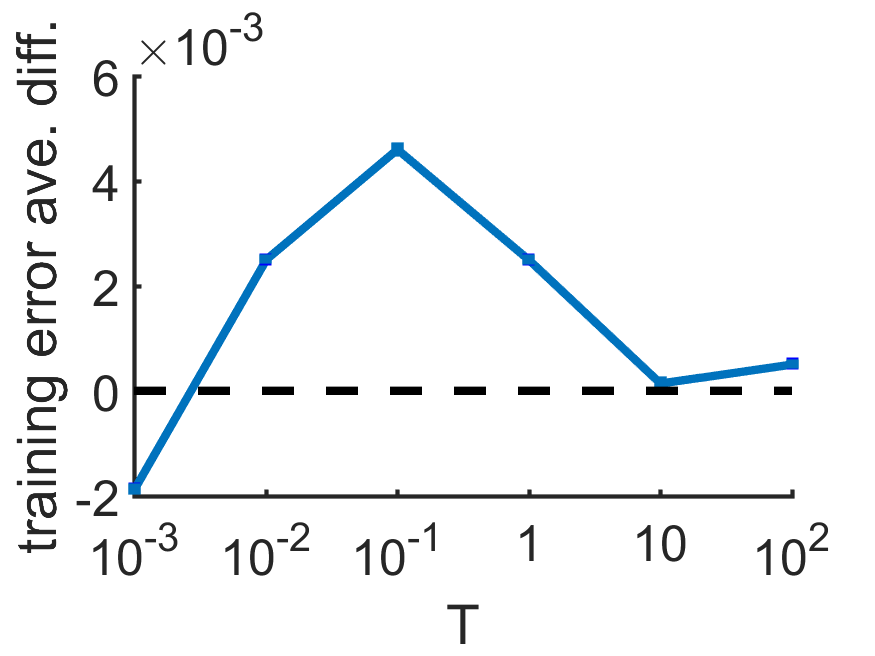}}
  \subfigure[p=4]{\includegraphics[width=0.22\linewidth]{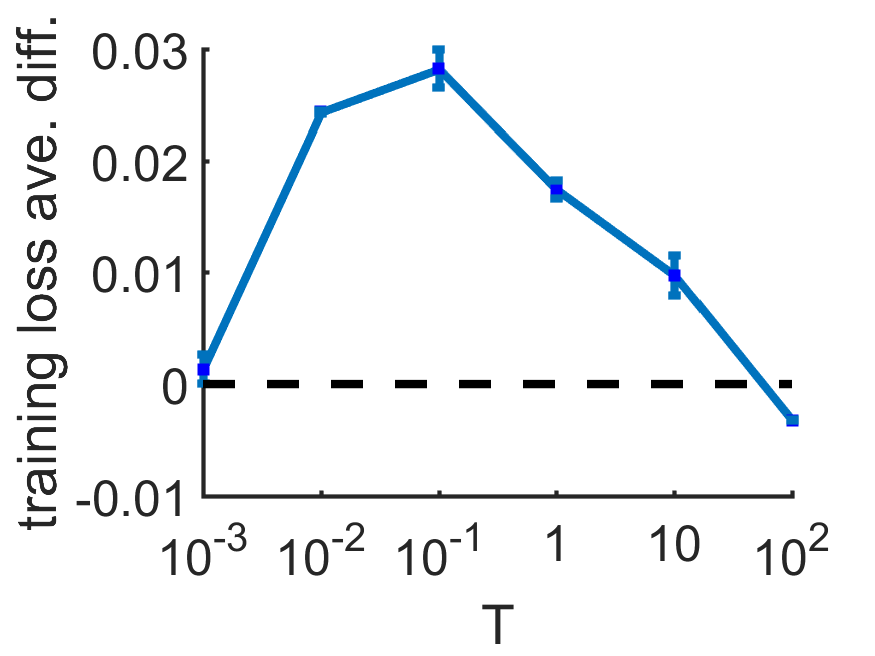}}\\
  \caption{Investigating the impact of $T$ where $T=\frac{1}{\tilde{a}}$. (a) and (b) represent the results on \textit{MNIST}, (c) and (d) represent the results on \textit{Fashion-MNIST}, (e) and (f) show the results on \textit{CIFAR-10}, (g) and (h) show the results on \textit{CIFAR-100}.} \label{fig:differentT}
\end{figure*}

Different values of $\tilde{a}$ denote different weighting strategies.
We try different values of $T$ where $T=\frac{1}{\tilde{a}}$ compared with the equally weighted case.
When $T \rightarrow 0$, it is equivalent to the sequential case, the result must be worse than the equally weighted one \cite{zinkevich2010parallelized}.
When $T \rightarrow \infty$, it means the weight value of each worker is almost equal to each other, the result will be the same as the equally weighted one.

In order to avoid outliers, we conduct five experiments with one epoch.
The point in \autoref{fig:differentT} is calculated by:
\begin{equation}\label{weightcomp}
    \frac{1}{5}\sum_{i=1}^5 \frac{\sum_{j=1}^N (\overline{v^j(jud)}-v_i^j(cur))}{N},
\end{equation}
where $N$ denotes the number of records in one experiment, $\overline{v^j(jud)}$ means the average baseline value of $j^{th}$ record and $v_i^j(cur)$ means the value of $j^{th}$ record for experiment $i$ .
We also plot the error bar of each point which is equal to the variance of $(\sum_{j=1}^N \frac{\left(\overline{ v^j(jud)}-v_1^j(cur)\right)}{N}, \dots , \sum_{j=1}^N \frac{\left(\overline{v^j(jud)}-
v_5^j(cur)\right)}{N})$
As the error bar is very small, this confirms the stability of our experiments.

As shown in \autoref{fig:differentT}, for \textit{MNIST} and \textit{CIFAR-10} datasets, the optimal $T$ is equal to one. 
For \textit{Fashion-MNIST} dataset, the optimal $T$ is 10.
Regarding \textit{CIFAR-100} dataset, the proper choice of $T$ is $10^{-1}$.

Due to the decreasing property of the curve compared to the baseline when $T \rightarrow 0$, we can also estimate the scope of the change
point that distinguishes the performance between the weighted case
and the baseline from \autoref{fig:differentT}.
For \textit{CIFAR-10,Fashion-MNIST} and \textit{MNIST}, when $T \leq 0.1$, the results begin to be worse than the baseline.
So the range of change point for these three datasets is around  $10^{-1}$.
As to \textit{CIFAR-100} dataset, the scope of change point is around $10^{-2}$.
We learn that as the complexity of the dataset increases, the range of change point approaches to zero.
\subsubsection{$\beta$ exploration}
$\beta$ determines how much we accept from the communication result.
When $\beta = 0$, it means reject the communication result which is equivalent to the sequential one and must be the worst case and $\beta=1$ denotes totally accepting the communication result.
In order to find the optimal $\beta$, we set $\beta=1$ as the baseline case and compare to other values of $\beta$.
To be fair, the experiments have been conducted five times with one epoch.
We use the same formula as Equation \eqref{weightcomp} to calculate the point in \autoref{fig:differentbeta}.
We also plot the error bar of each point in the same way above.

As shown in \autoref{fig:differentbeta}, for \textit{MNIST} and \textit{CIFAR-10} datasets, the optimal $\beta$ is equal to 0.9. 
For \textit{CIFAR-100}, the proper choice is 0.8 and for \textit{Fashion-MNIST}, the optimal $\beta$ is 0.7.
As $\beta$ decreases, the results on all four datasets become worse than the baseline case.
Since the error bar is very small, the stability of our experiments is guaranteed.
\begin{figure*}[h]
  \center
  \subfigure[p=8]{\includegraphics[width=0.22\linewidth]{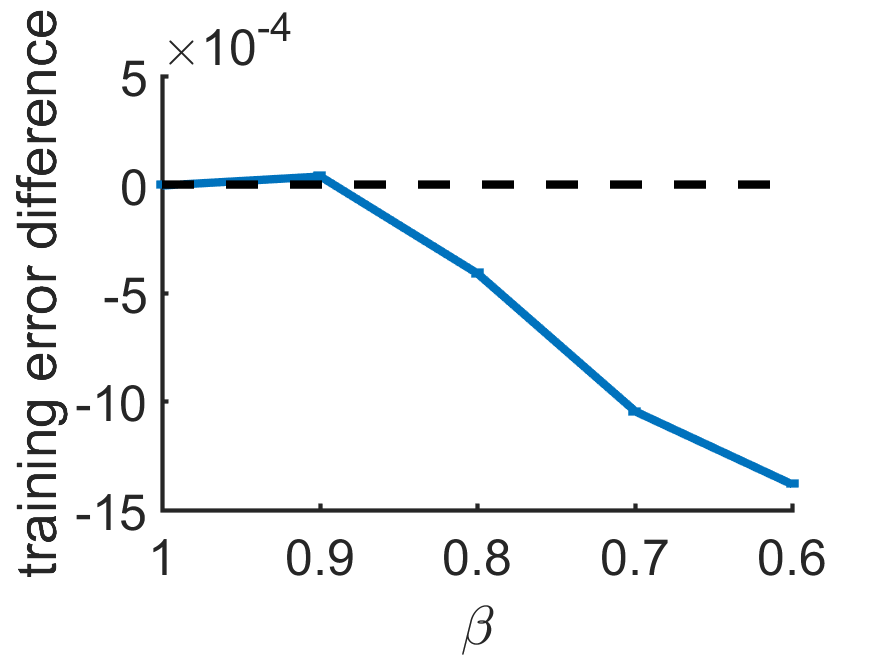}}
  \subfigure[p=8]{\includegraphics[width=0.22\linewidth]{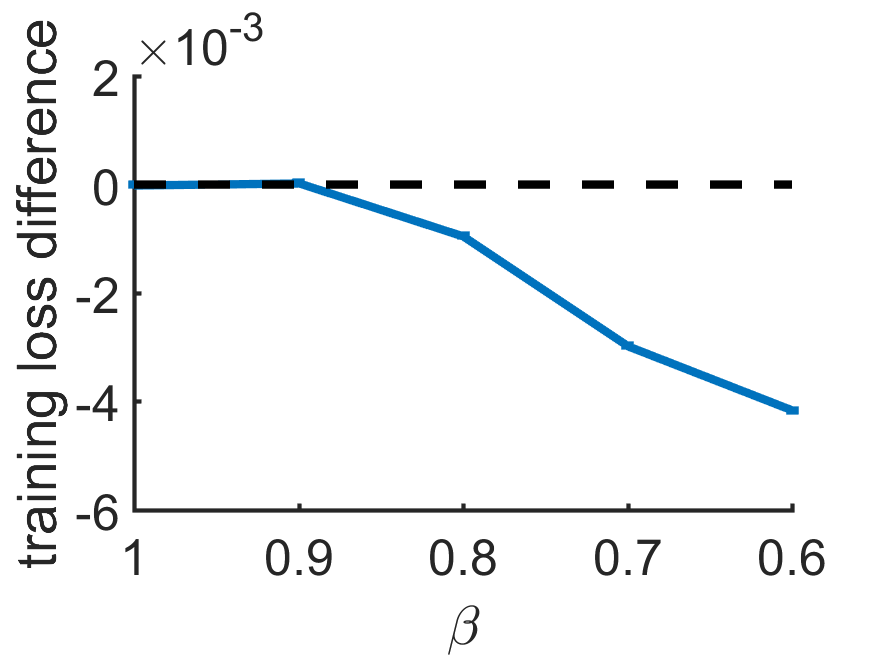}}
  \subfigure[p=8]{\includegraphics[width=0.22\linewidth]{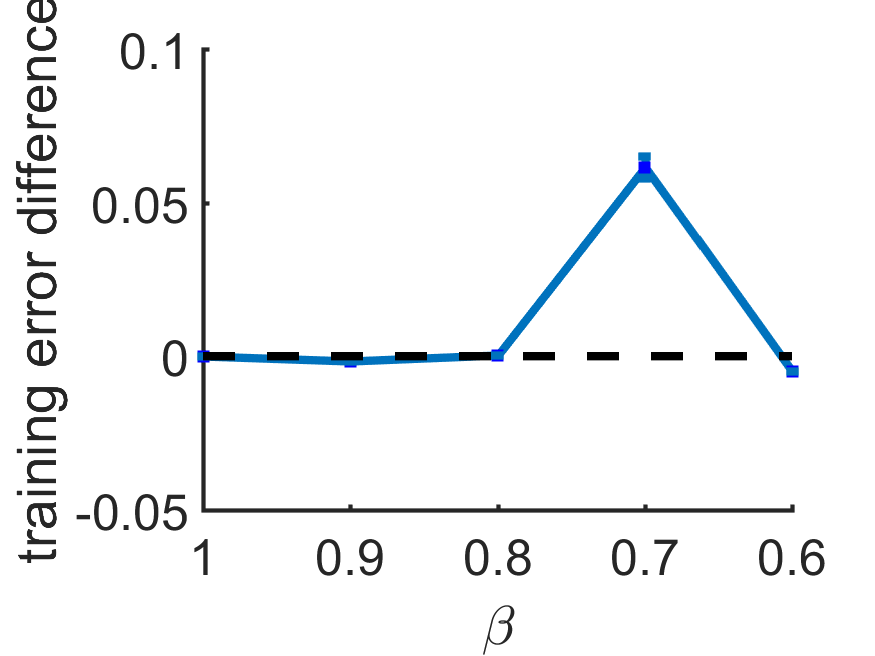}}
  \subfigure[p=8]{\includegraphics[width=0.22\linewidth]{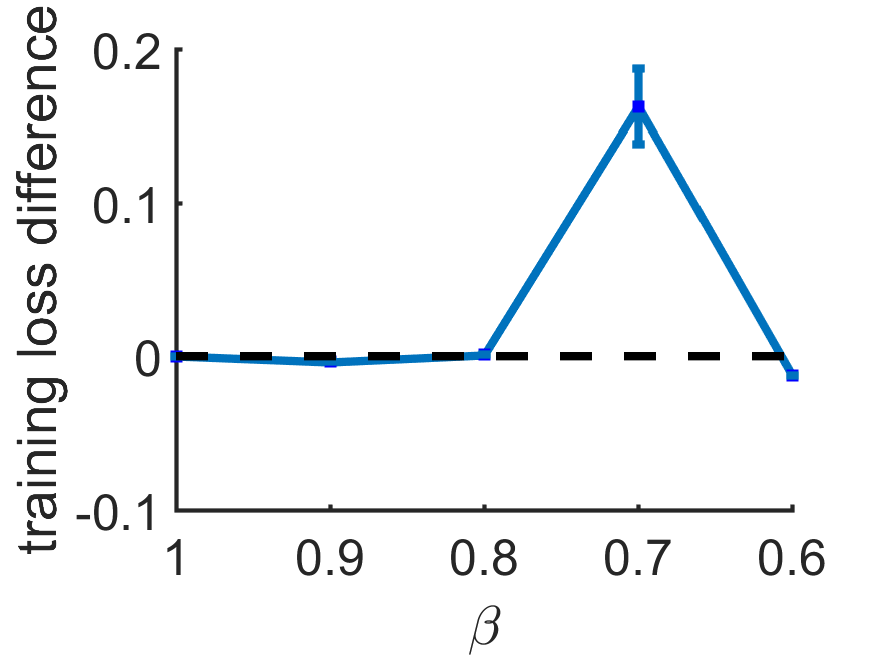}}\\
  \subfigure[p=4]{\includegraphics[width=0.22\linewidth]{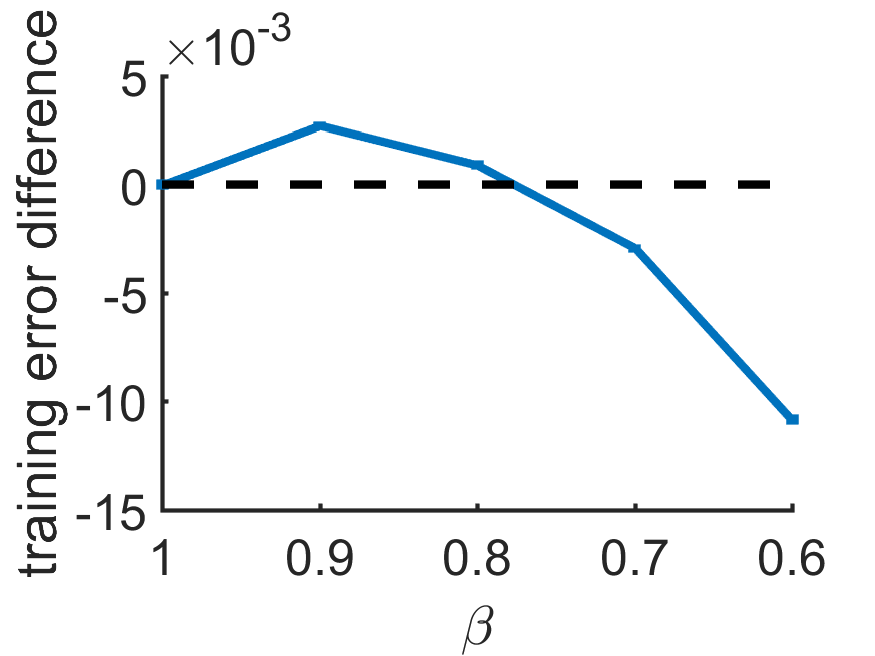}}
  \subfigure[p=4]{\includegraphics[width=0.22\linewidth]{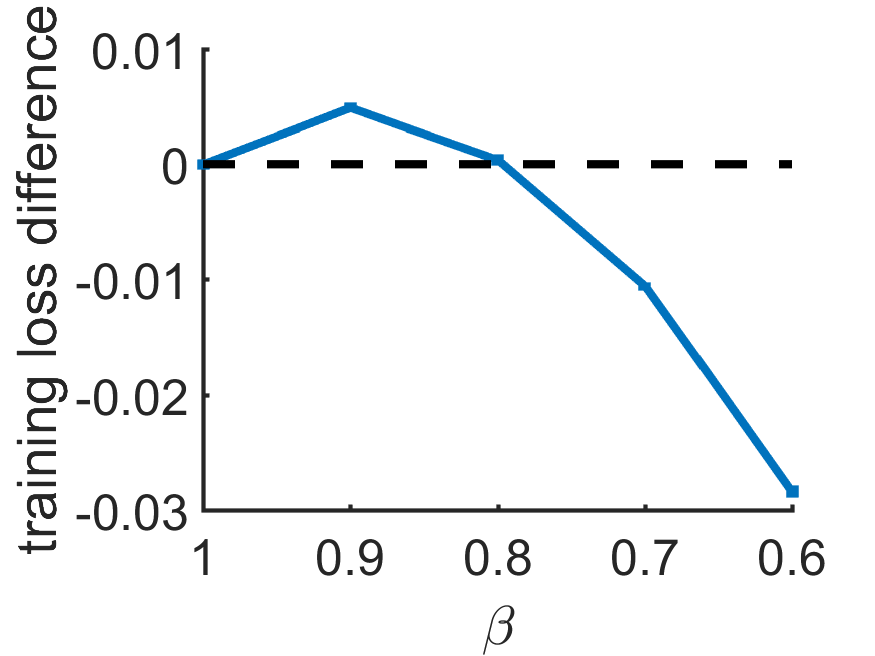}}
  \subfigure[p=4]{\includegraphics[width=0.22\linewidth]{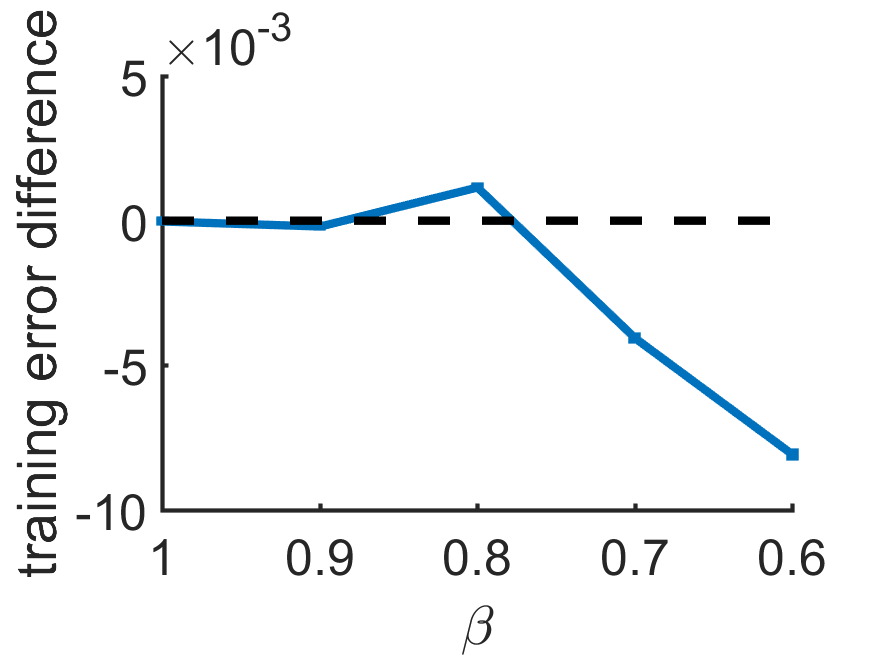}}
  \subfigure[p=4]{\includegraphics[width=0.22\linewidth]{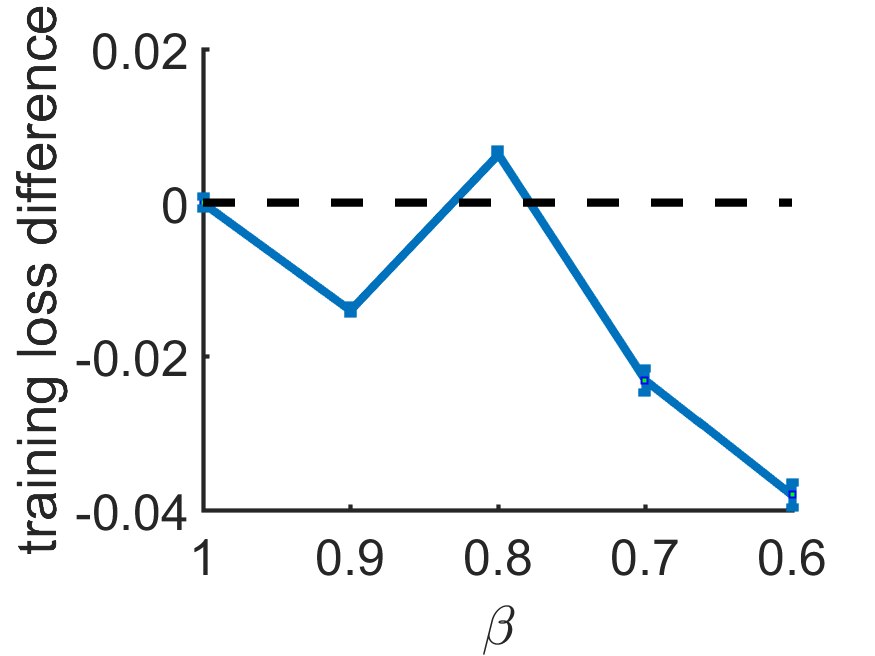}}\\
  \caption{Investigating the impact of $\beta$. (a) and (b) represent the results on \textit{MNIST}, (c) and (d) represent the results on \textit{Fashion-MNIST}, (e) and (f) show the results on \textit{CIFAR-10}, (g) and (h) show the results on \textit{CIFAR-100}.} \label{fig:differentbeta}
\end{figure*}
\subsubsection{Choice of estimation sample size $m$}
Estimation sample size $m$ is also an important parameter that needs to be adjusted.
If $m$ is small, local workers will get more freedom to explore, but the accuracy of the estimated weight will be affected.
If $m$ is large, the accuracy of the estimated weight can be guaranteed; however, the efficiency of parallelization will be weakened.
So we need to balance the trade off between the accuracy of estimation and the efficiency of parallel searching.

We have tried different values of $m$ and found the best choice. 
As shown in \autoref{fig:CIFARes}, for $m=1, 10$, although they allow more explorations than $m=100,1000$, their accuracy of estimation is low and unstable.
\begin{figure}[t]
  \centering
  \subfigure[p=2]{\includegraphics[width=0.25\linewidth]{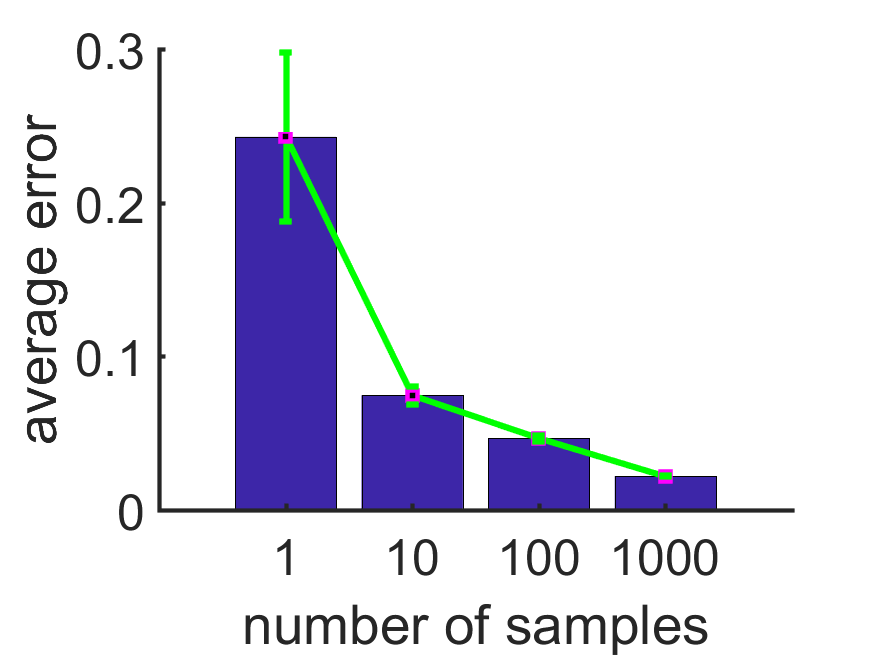}}
  \subfigure[p=4]{\includegraphics[width=0.25\linewidth]{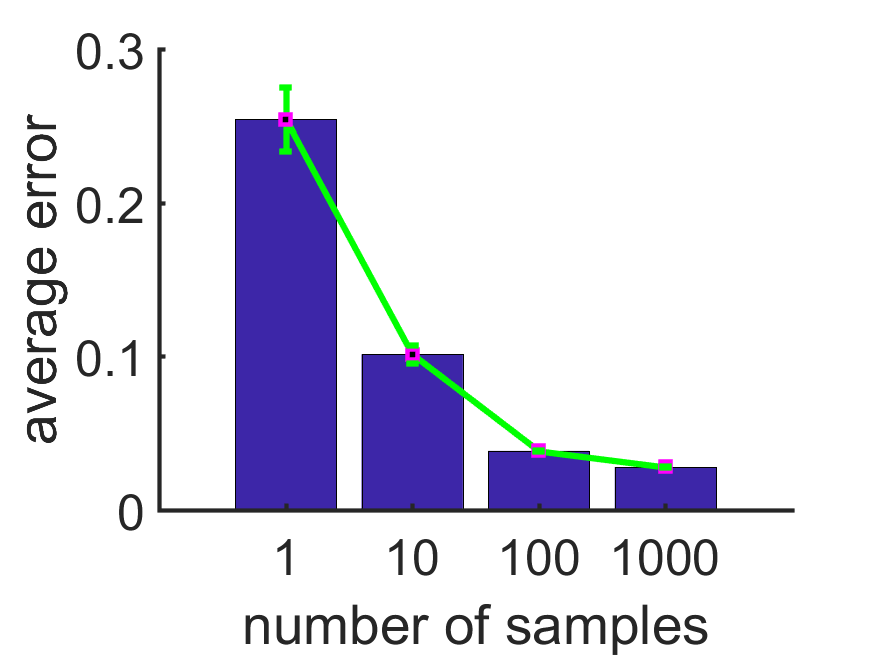}}
  \subfigure[p=8]{\includegraphics[width=0.25\linewidth]{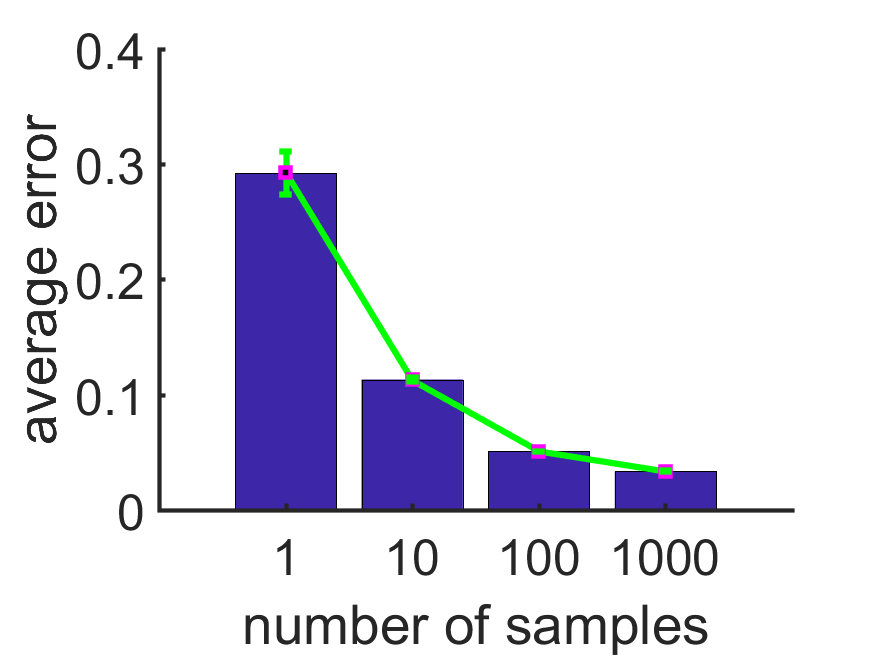}}
  \caption{Estimation accuracy of weight.}\label{fig:CIFARes}
\end{figure}
Comparing $m=100$ and $m=1000$, they both perform well in terms of estimating the accuracy, we choose $m=100$ for maintaining the efficiency of parallel computing.
\subsection{Performances on different $\tau$}
\begin{figure}[h]
  \centering
  \subfigure[p=2]{\includegraphics[width=0.25\linewidth]{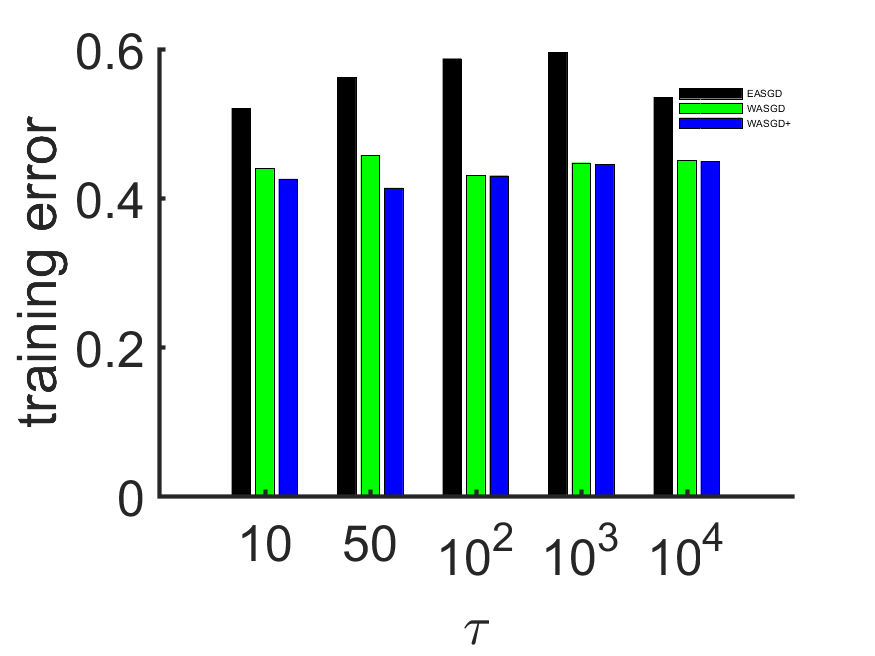}}
  \subfigure[p=2]{\includegraphics[width=0.25\linewidth]{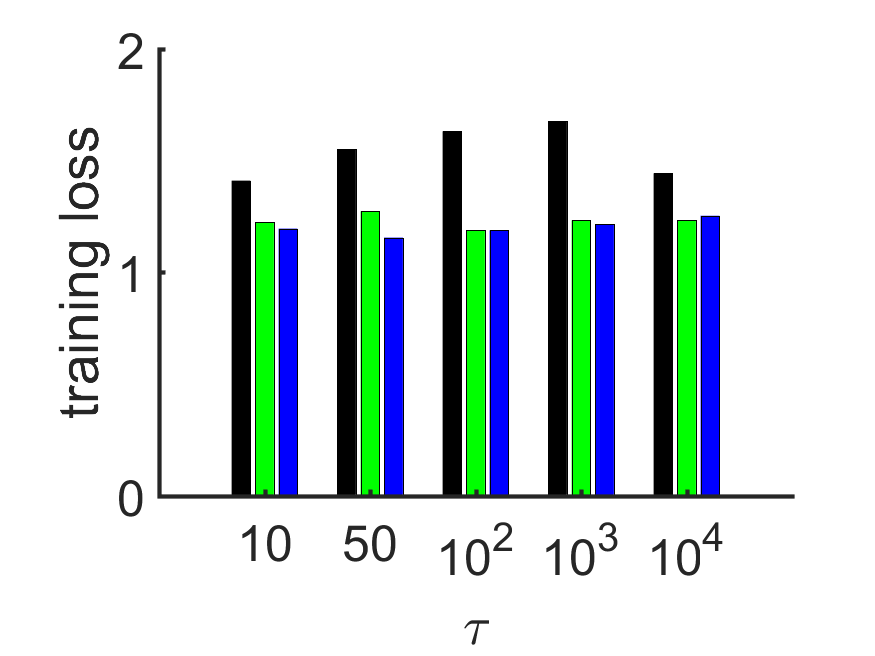}}\\
  \subfigure[p=4]{\includegraphics[width=0.25\linewidth]{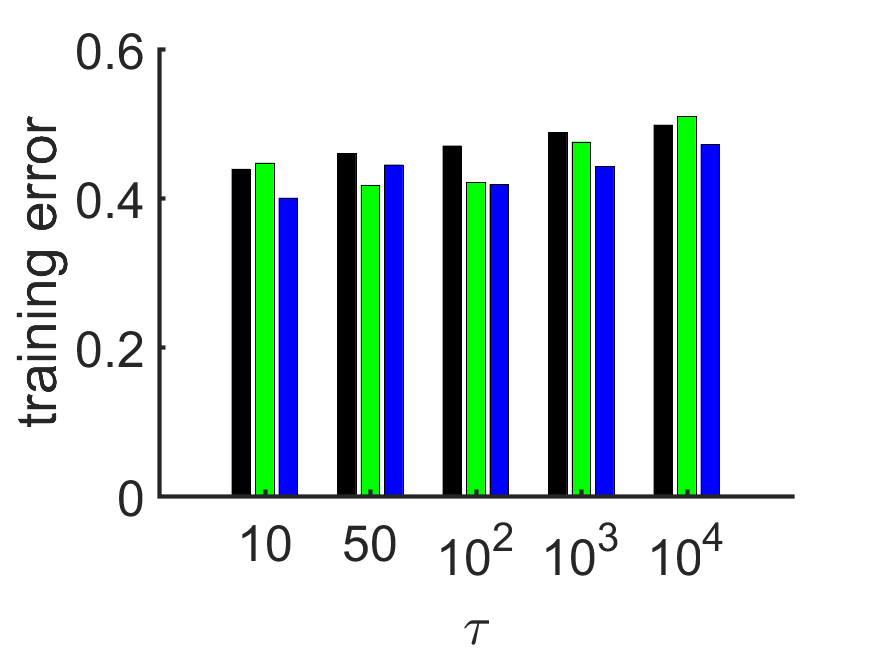}}
  \subfigure[p=4]{\includegraphics[width=0.25\linewidth]{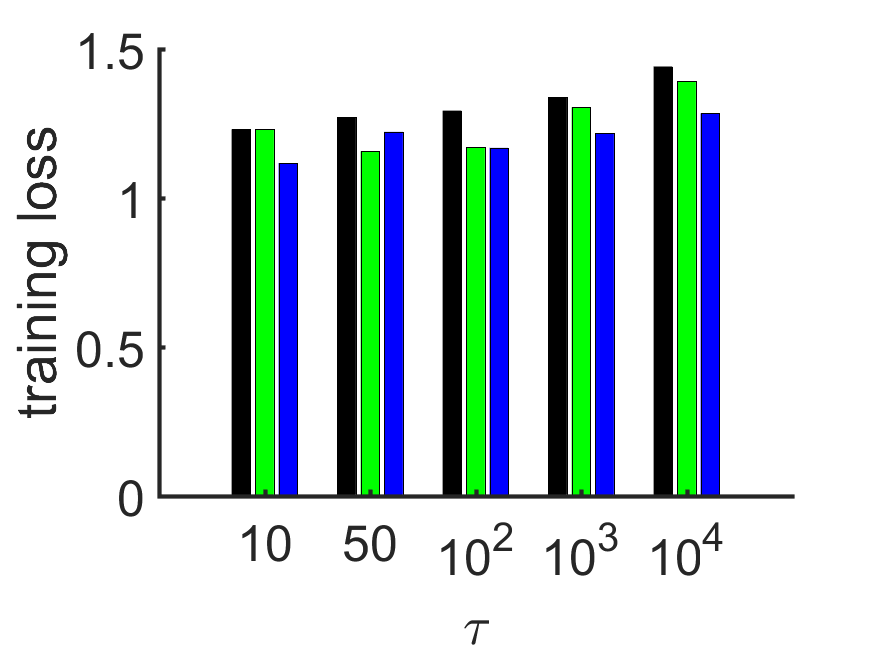}}\\
  \subfigure[p=8]{\includegraphics[width=0.25\linewidth]{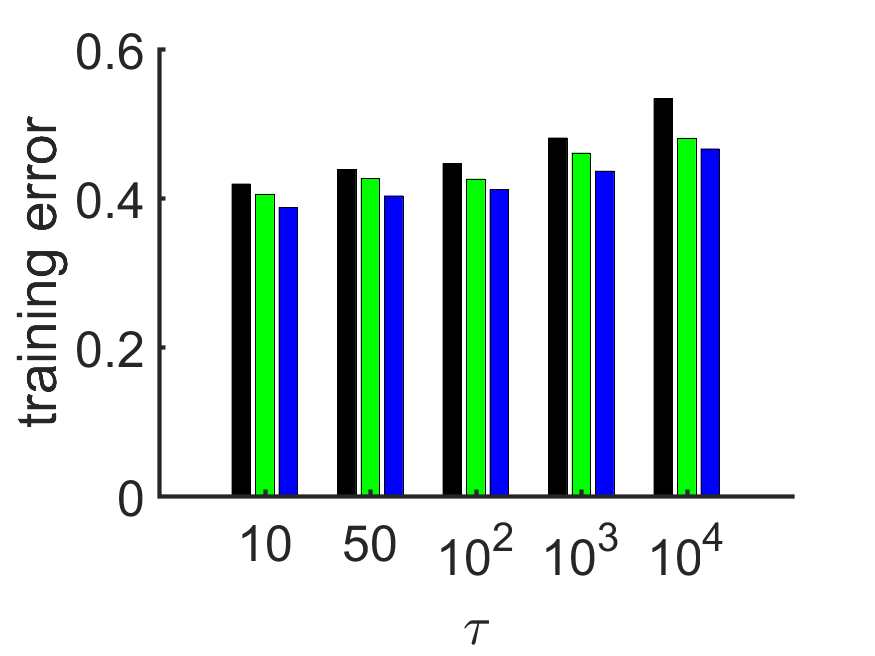}}
  \subfigure[p=8]{\includegraphics[width=0.25\linewidth]{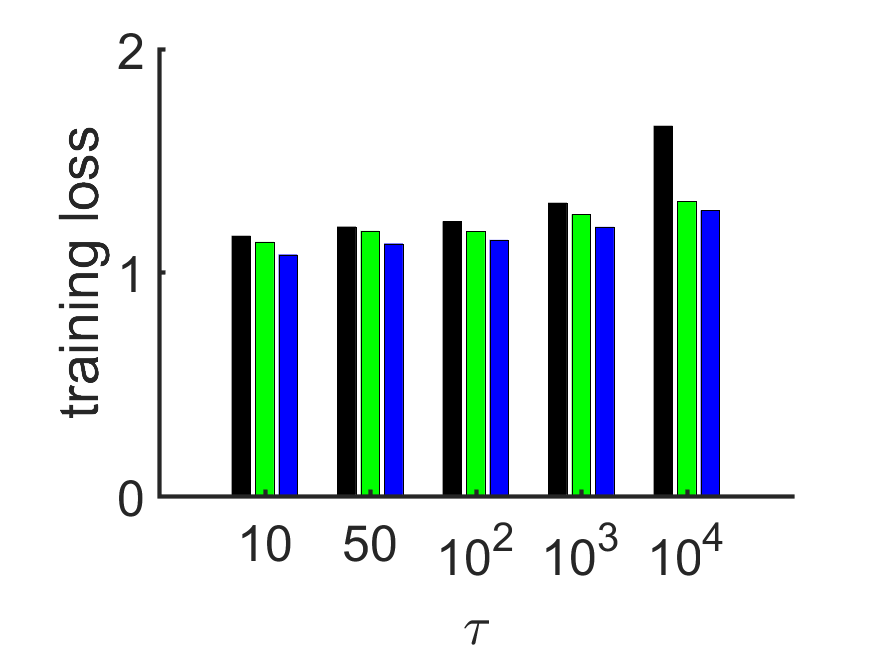}}
  \caption{Positions after two epochs of different $\tau$ on \textit{CIFAR-10}.}\label{fig:cifartau}
\end{figure}
We have multiple choices for the communication period $\tau$, smaller $\tau$ might improve the convergence rate but spends too much time on communicating. 
Larger $\tau$ will weaken the effect of parallel computing.
So we want to explore the performance on different $\tau$ and an optimal one to balance between time and convergence rate.
\begin{figure*}[!t]
  \centering
  \subfigure[p=2]{\includegraphics[width=0.22\linewidth]{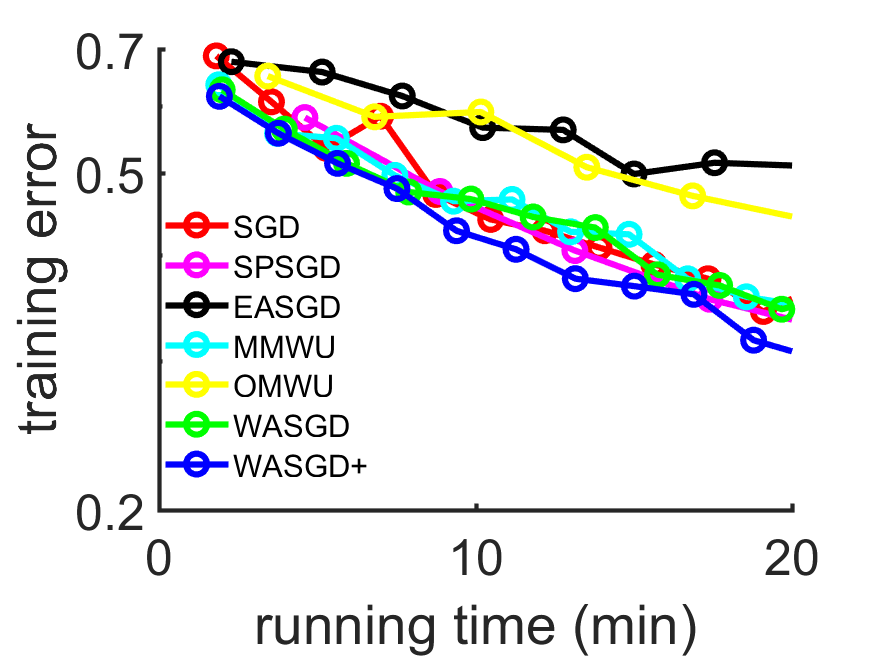}}
  \subfigure[p=2]{\includegraphics[width=0.22\linewidth]{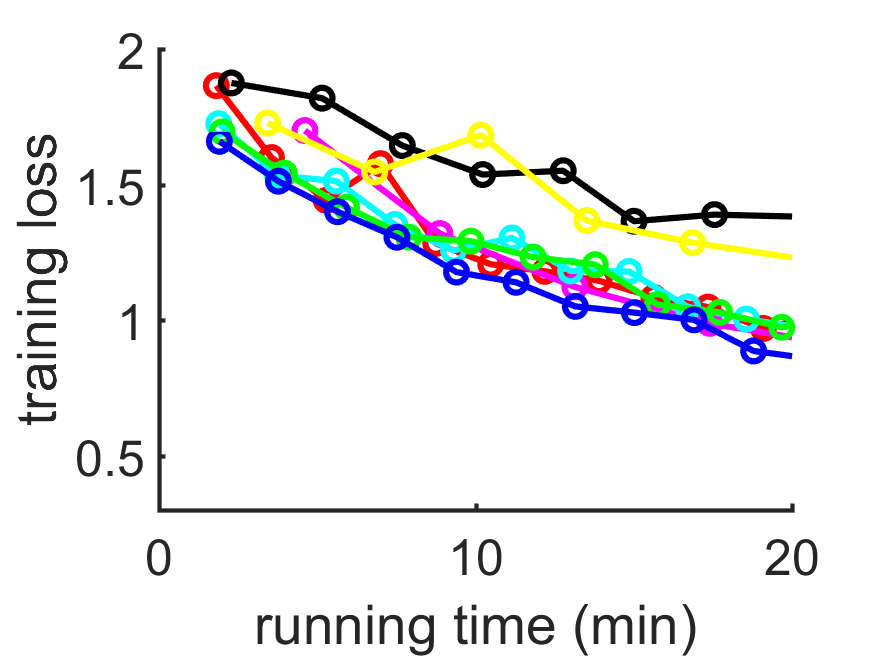}}
  \subfigure[p=2]{\includegraphics[width=0.22\linewidth]{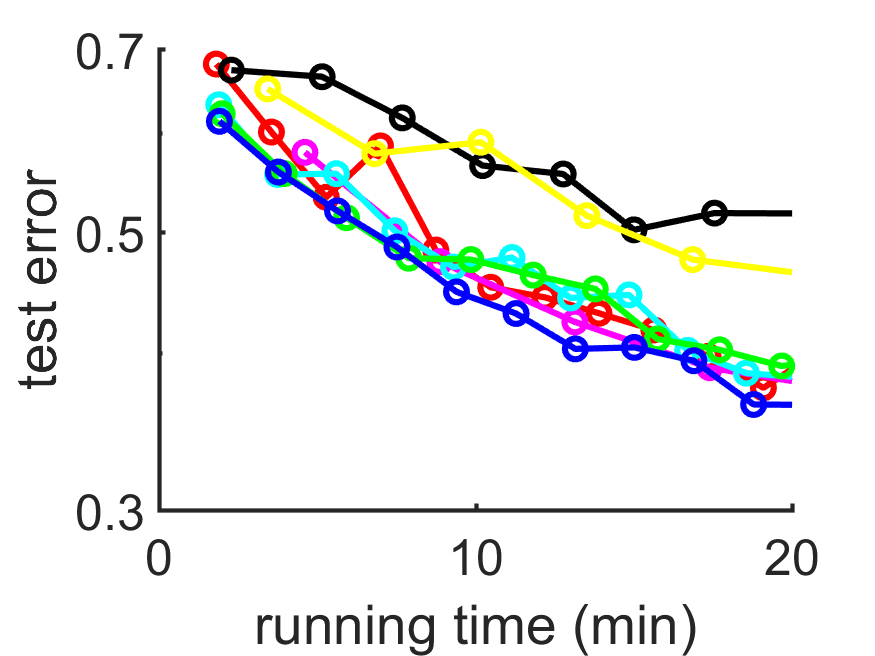}}
  \subfigure[p=2]{\includegraphics[width=0.22\linewidth]{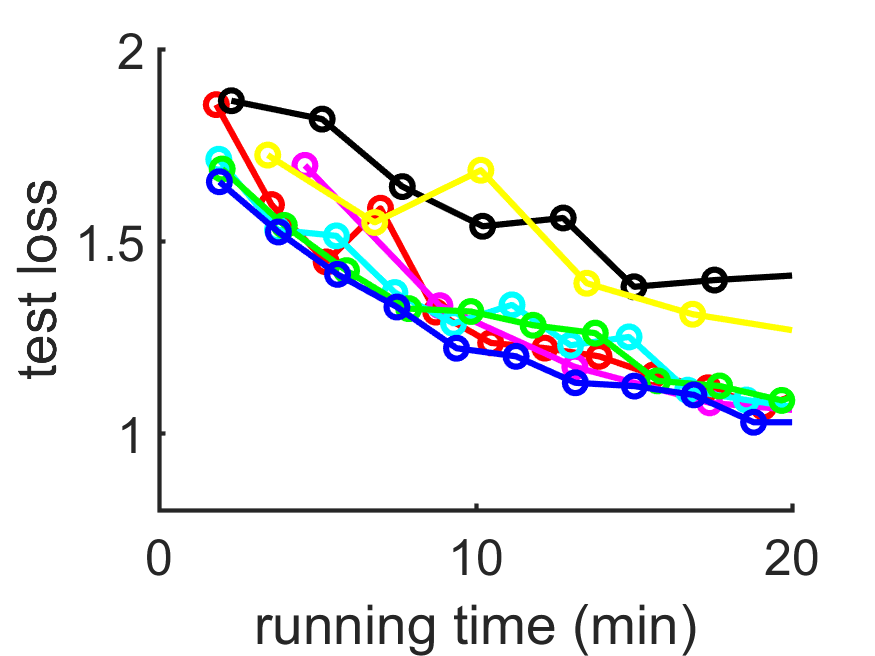}}\\
  \subfigure[p=4]{\includegraphics[width=0.22\linewidth]{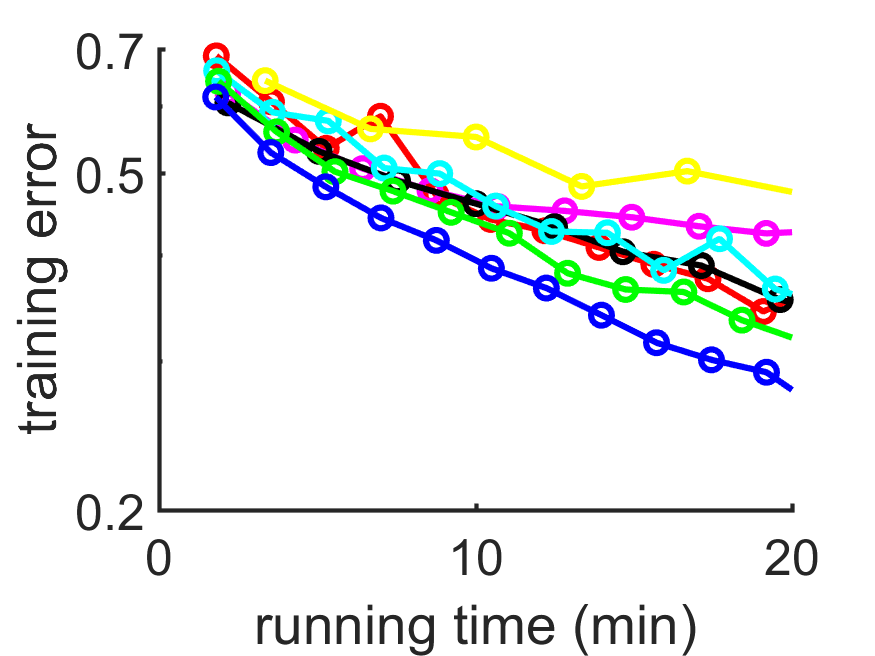}}
  \subfigure[p=4]{\includegraphics[width=0.22\linewidth]{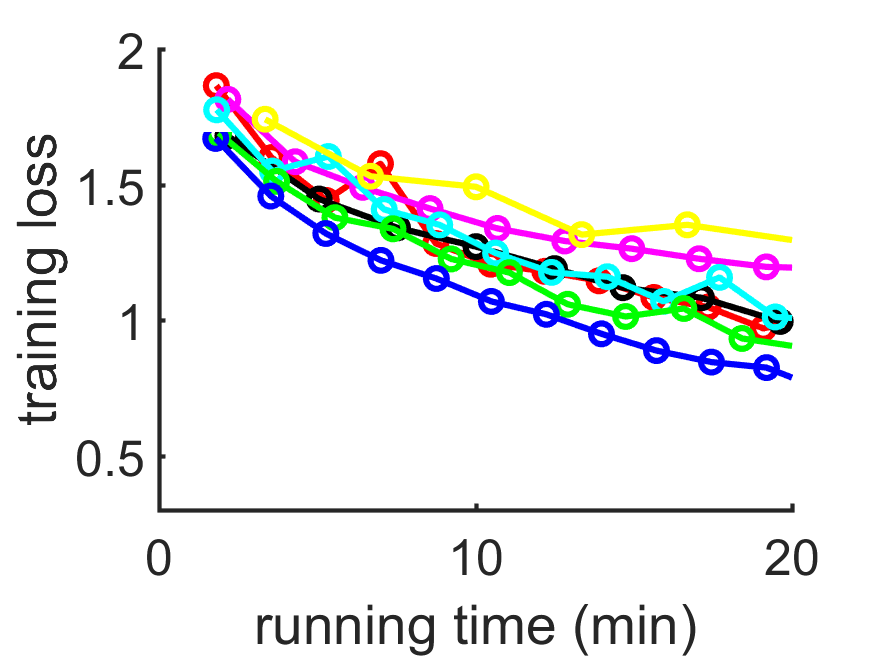}}
  \subfigure[p=4]{\includegraphics[width=0.22\linewidth]{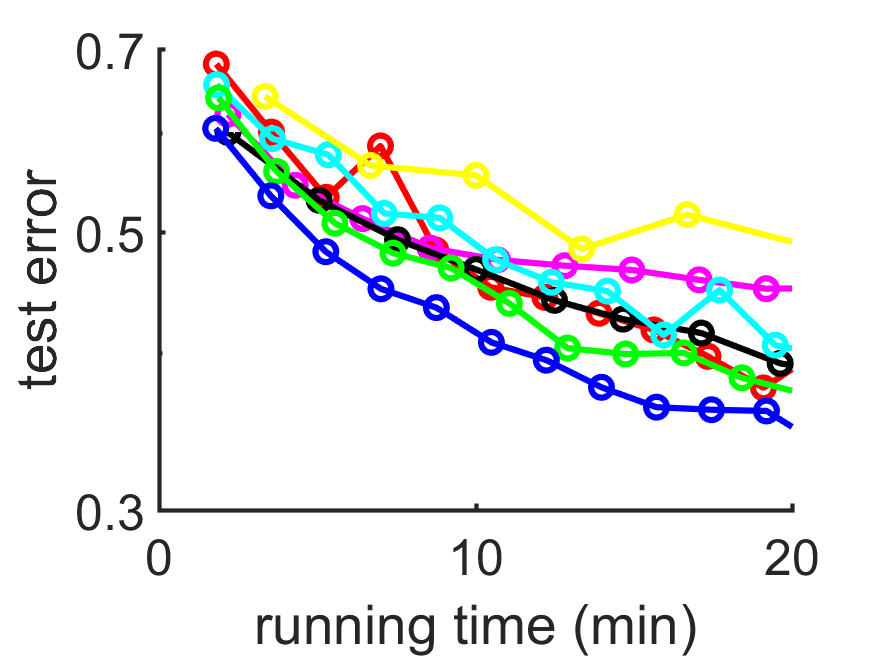}}
  \subfigure[p=4]{\includegraphics[width=0.22\linewidth]{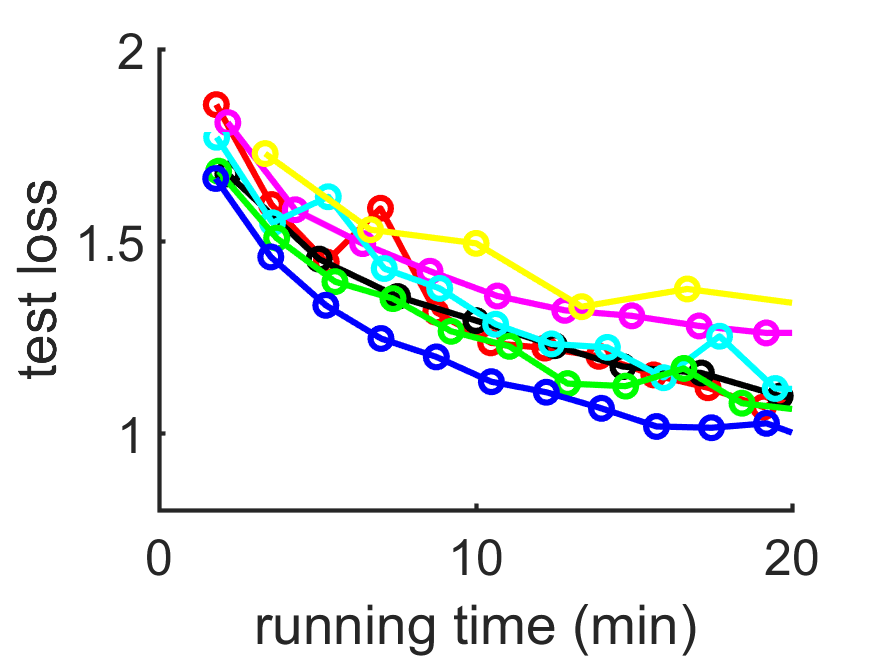}}\\
 \subfigure[p=8]{\includegraphics[width=0.22\linewidth]{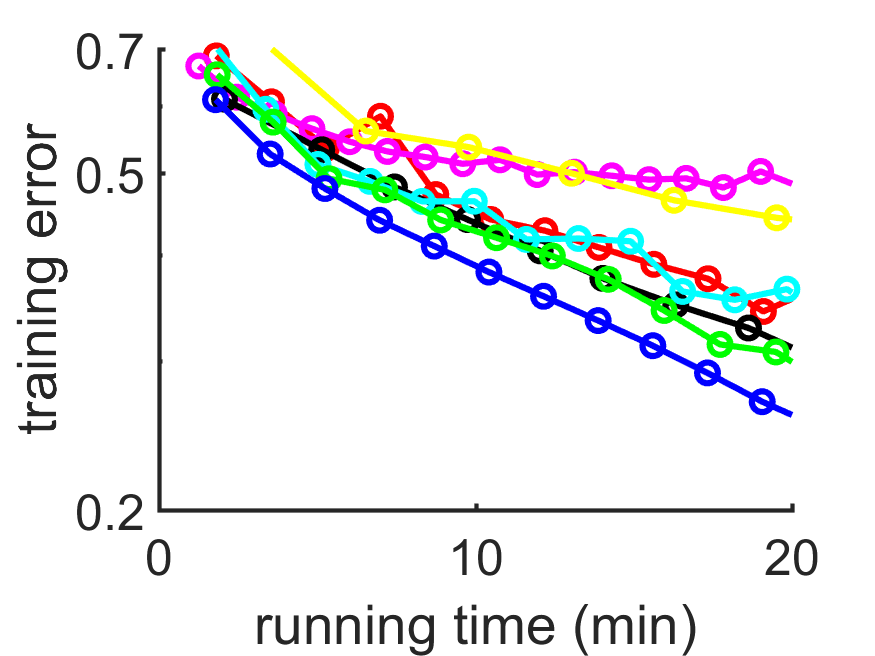}}
  \subfigure[p=8]{\includegraphics[width=0.22\linewidth]{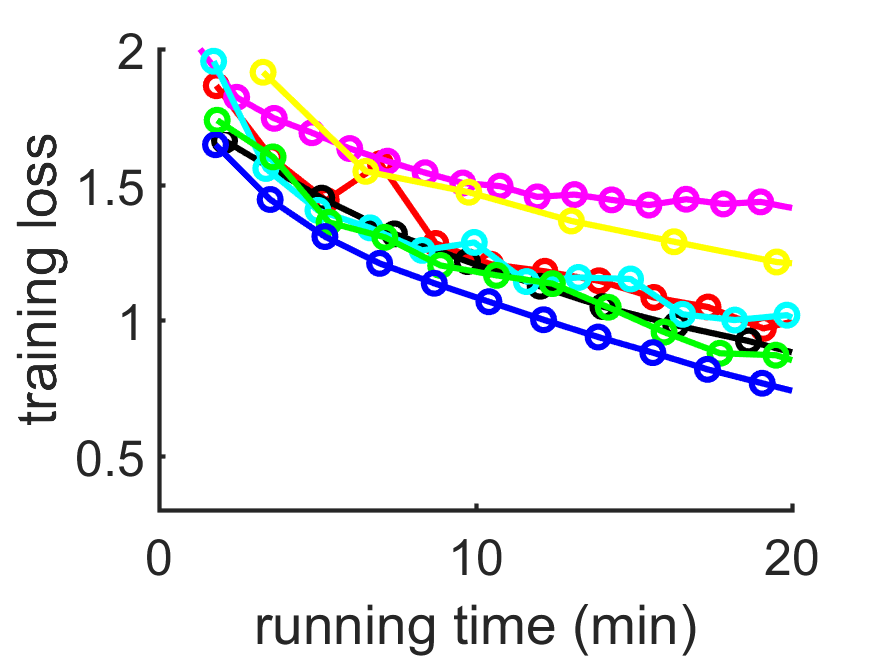}}
  \subfigure[p=8]{\includegraphics[width=0.22\linewidth]{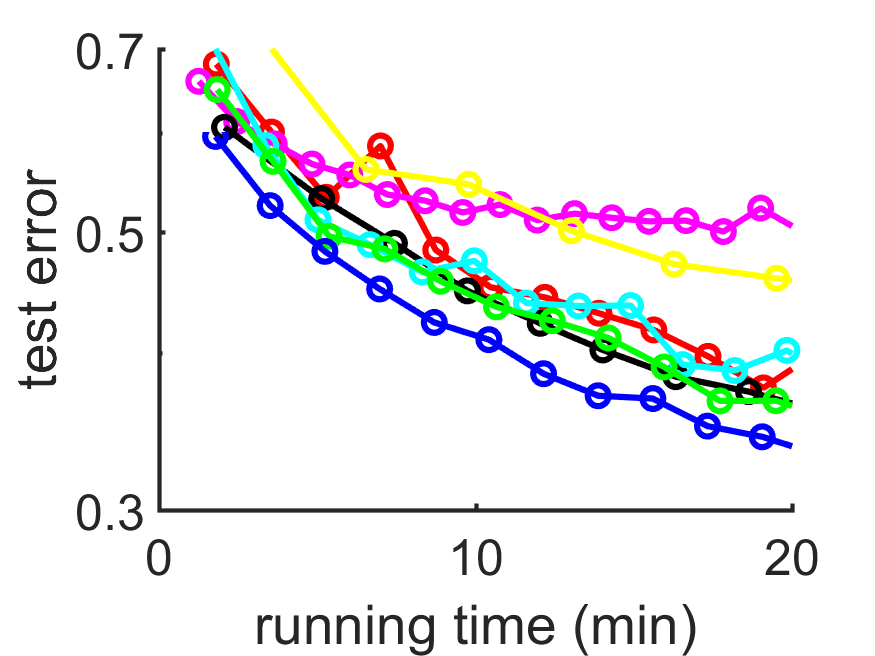}}
  \subfigure[p=8]{\includegraphics[width=0.22\linewidth]{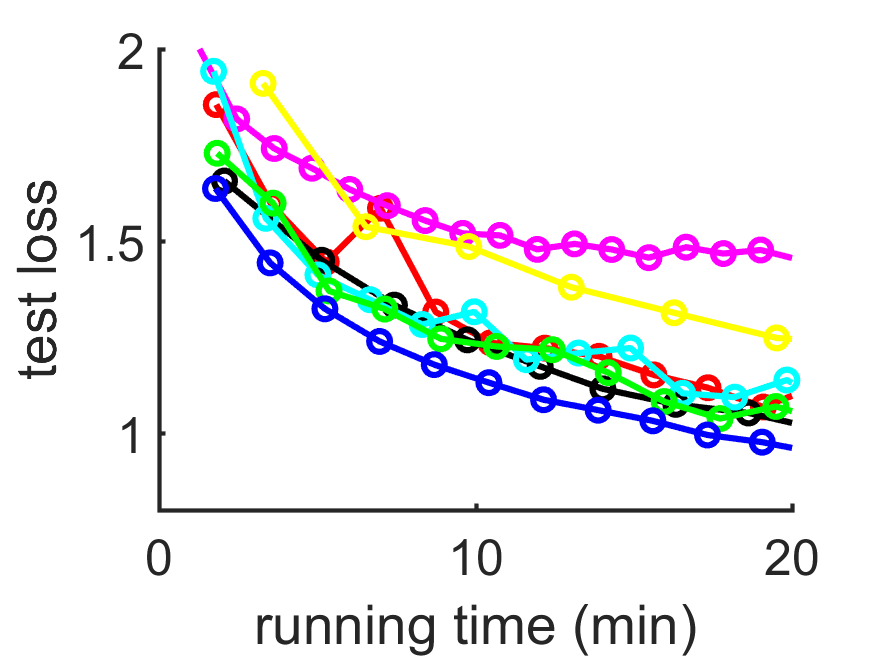}}\\
  \caption{Experiments on CIFAR-10.}\label{fig:cifar10mainex}
\end{figure*}
The author of \textit{EASGD} has tried $\tau = \{1,4,16,64\}$ and find the optimal $\tau = \{16,64\}$.
The tendency to achieve better test performance with larger $\tau$ is a strength for the \textit{EASGD} algorithm.
According to this characteristic, we tested the communication periods from the following set $\tau = \{10^1,50,10^2,10^3,10^4\}$ for \textit{EASGD}, \textit{WASGD} and \textit{WASGD+}.
In order to guarantee the fairness of the results, we compare the performance after two epochs of training.
\autoref{fig:cifartau} depicts the performance of two methods under different processors.
As shown in the picture, when $\tau$ and $p$ are the same, \textit{WASGD} can achieve better performance than \textit{EASGD}.
\textit{WASGD+} is outperforming all other benchmarks.
We also find that $\tau=1000$ in \textit{WASGD+} achieves almost the same performance as $\tau=50$ in \textit{EASGD} under the same $p$.
\begin{figure*}[!t]
  \centering
  \subfigure[p=2]{\includegraphics[width=0.21\linewidth]{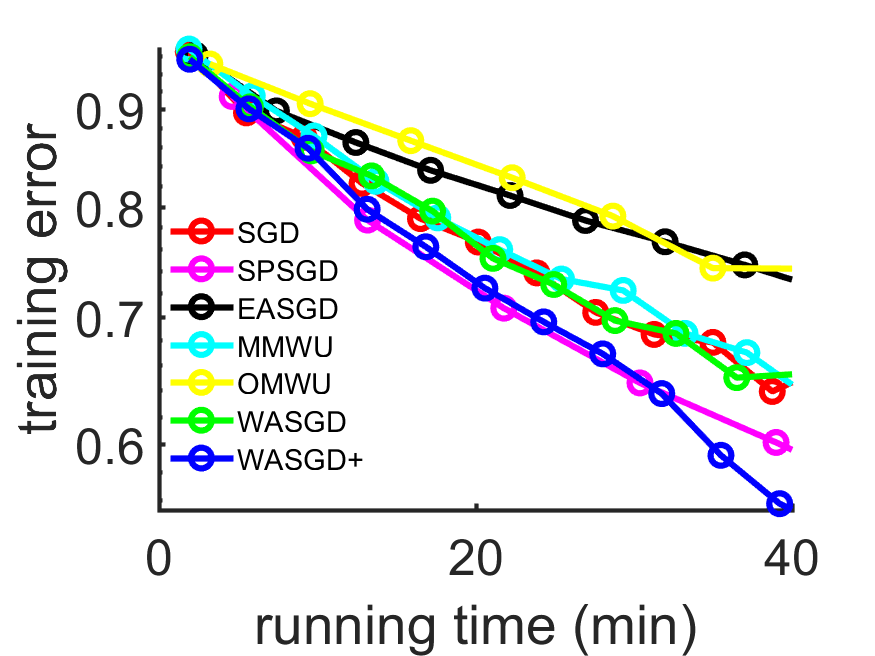}}
  \subfigure[p=2]{\includegraphics[width=0.21\linewidth]{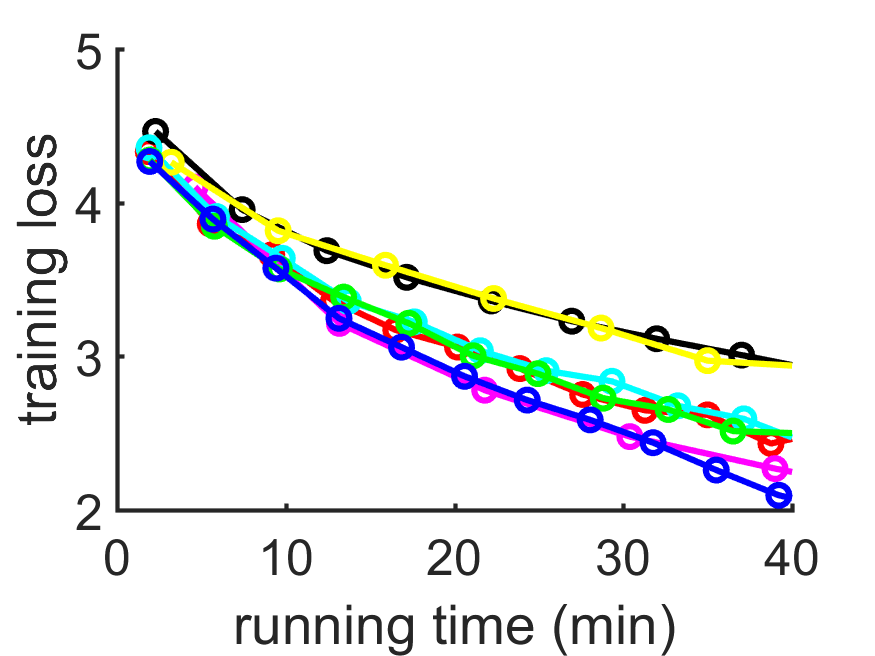}}
  \subfigure[p=2]{\includegraphics[width=0.21\linewidth]{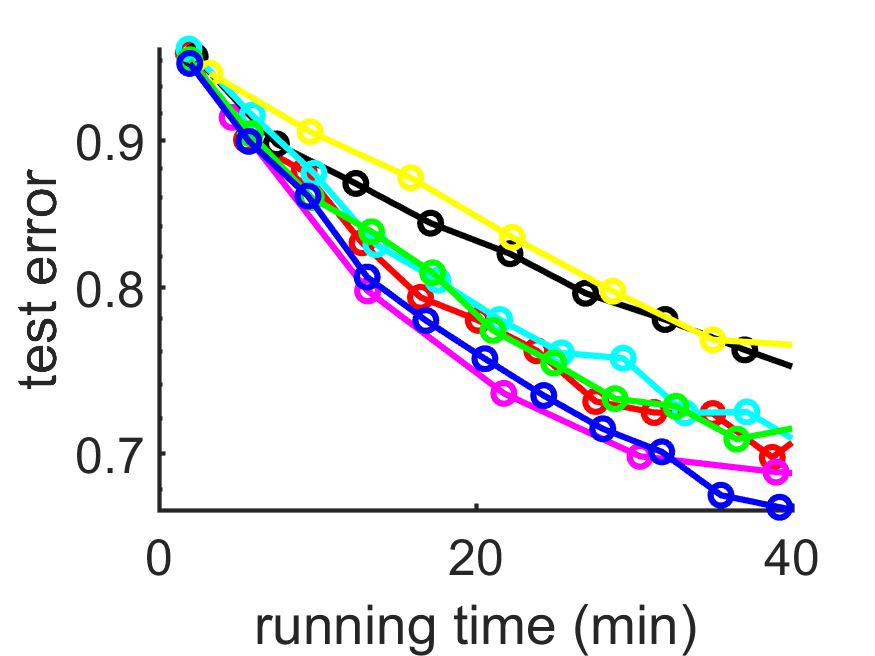}}
  \subfigure[p=2]{\includegraphics[width=0.21\linewidth]{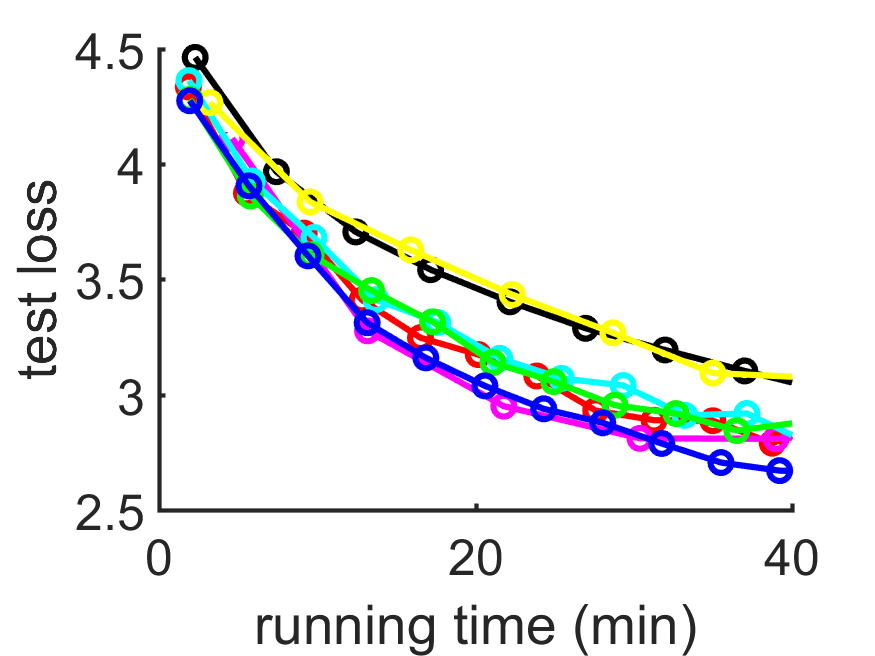}}\\
  \subfigure[p=4]{\includegraphics[width=0.21\linewidth]{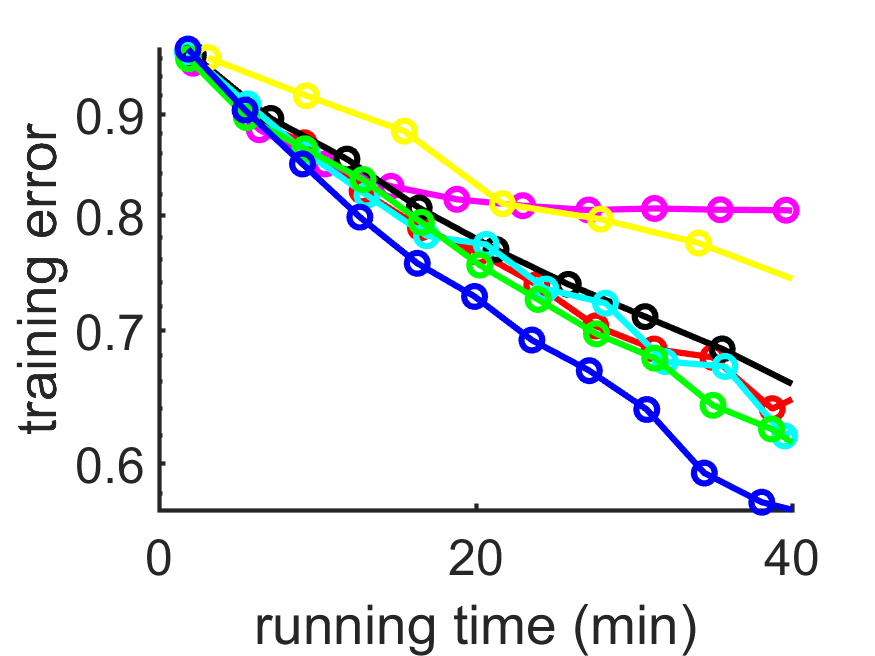}}
  \subfigure[p=4]{\includegraphics[width=0.21\linewidth]{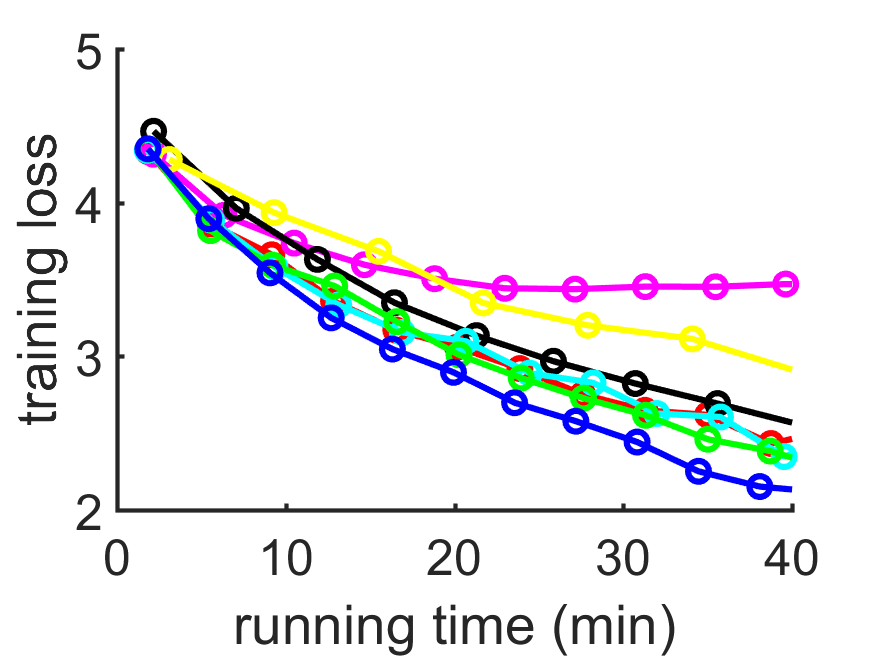}}
  \subfigure[p=4]{\includegraphics[width=0.21\linewidth]{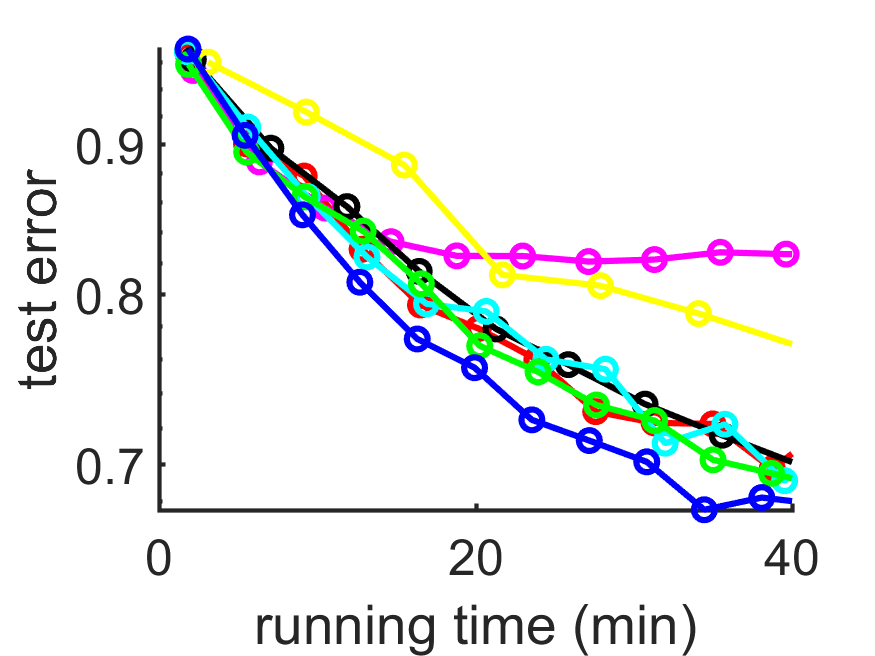}}
  \subfigure[p=4]{\includegraphics[width=0.21\linewidth]{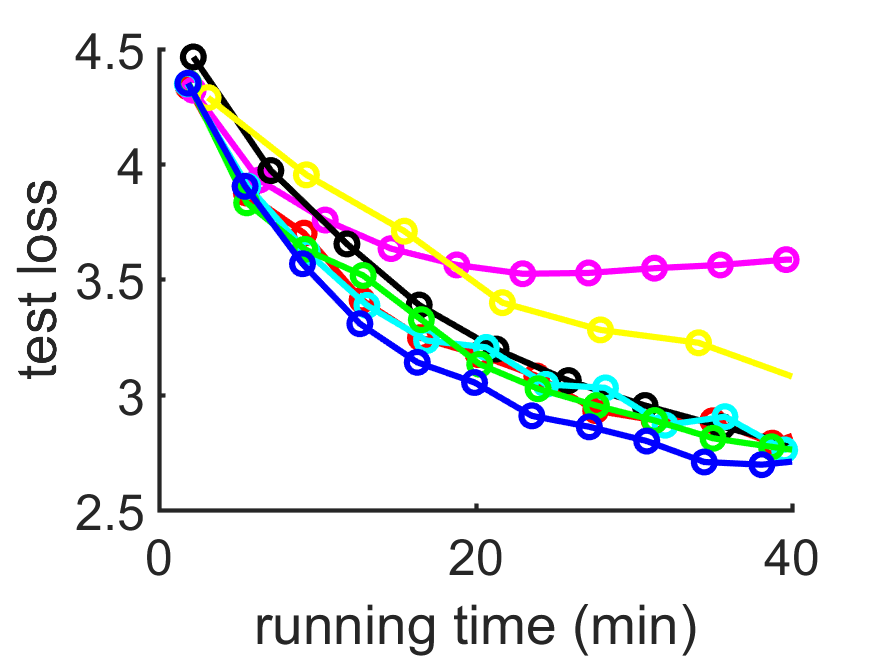}}\\
 \subfigure[p=8]{\includegraphics[width=0.21\linewidth]{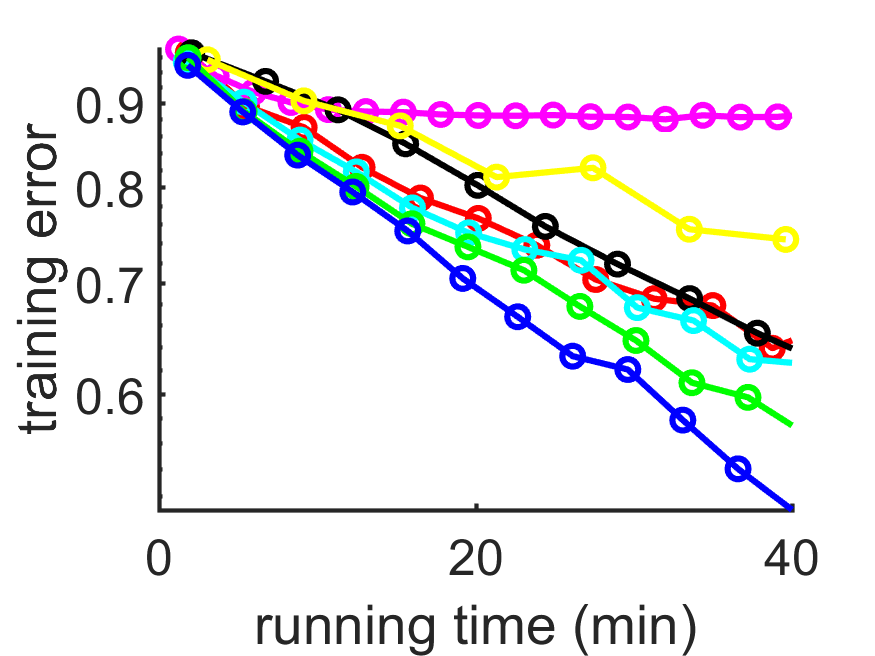}}
  \subfigure[p=8]{\includegraphics[width=0.21\linewidth]{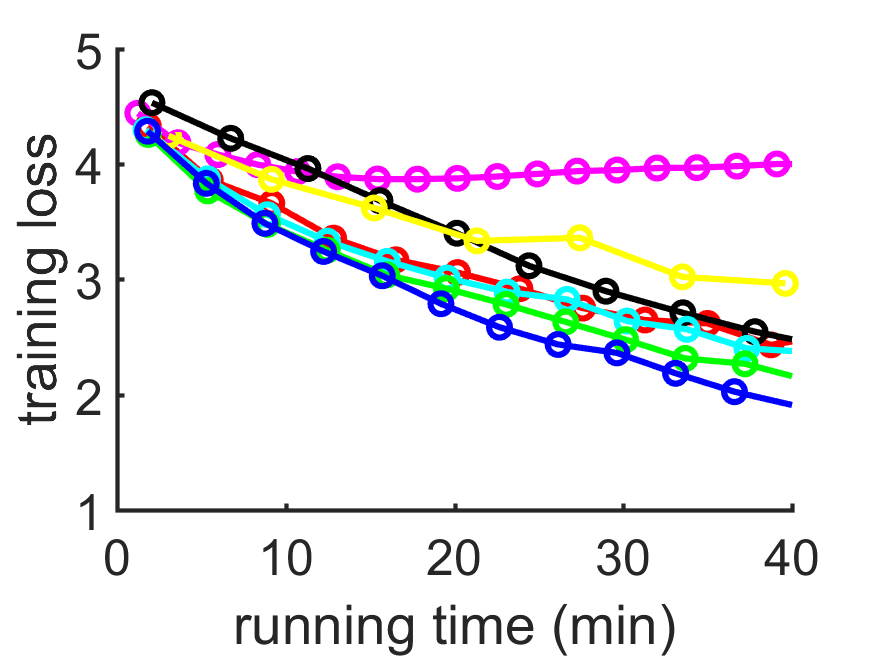}}
  \subfigure[p=8]{\includegraphics[width=0.21\linewidth]{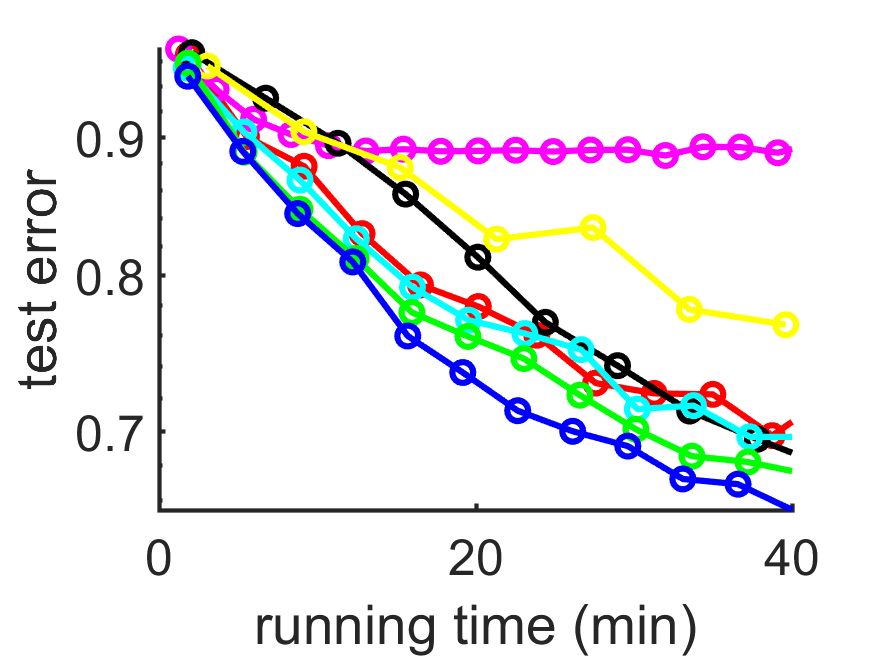}}
  \subfigure[p=8]{\includegraphics[width=0.21\linewidth]{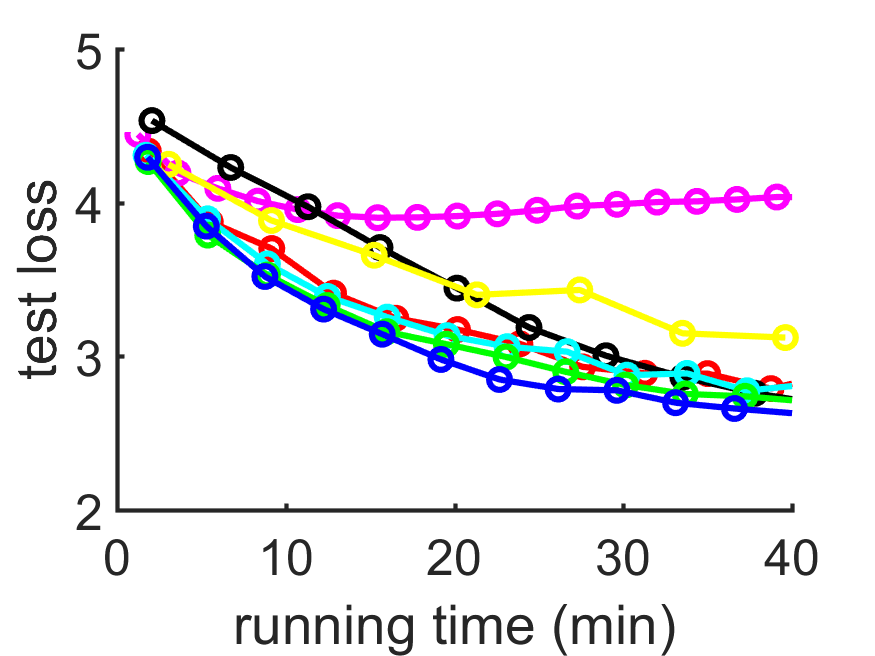}}\\
  \caption{Experiments on CIFAR-100.}\label{fig:cifar100mainex}
\end{figure*}
\subsection{Results}
\begin{figure}[tp!]
  \centering
  \subfigure[p=4]{\includegraphics[width=0.25\linewidth]{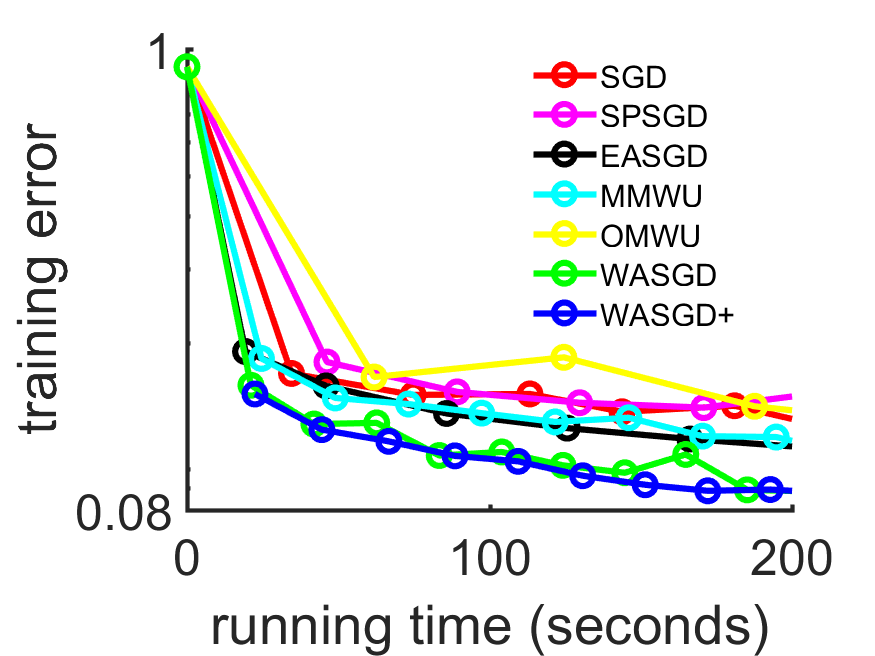}}
  \subfigure[p=4]{\includegraphics[width=0.25\linewidth]{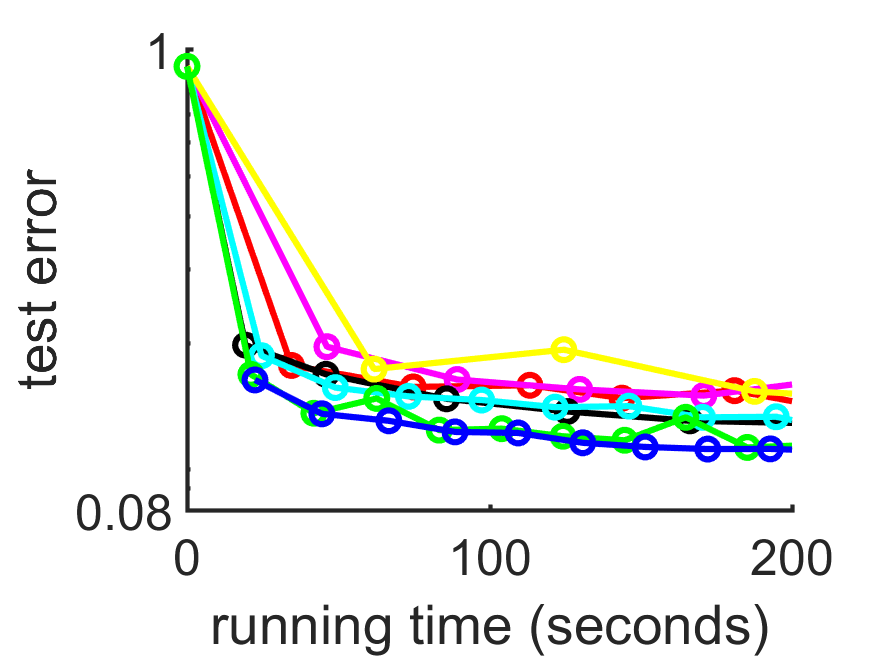}}\\
   \subfigure[p=8]{\includegraphics[width=0.25\linewidth]{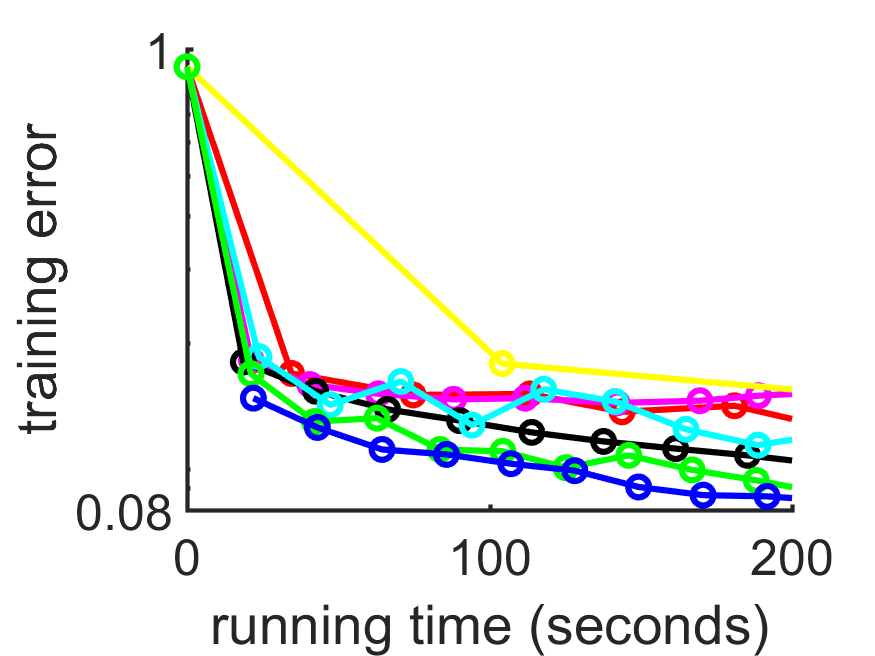}}
    \subfigure[p=8]{\includegraphics[width=0.25\linewidth]{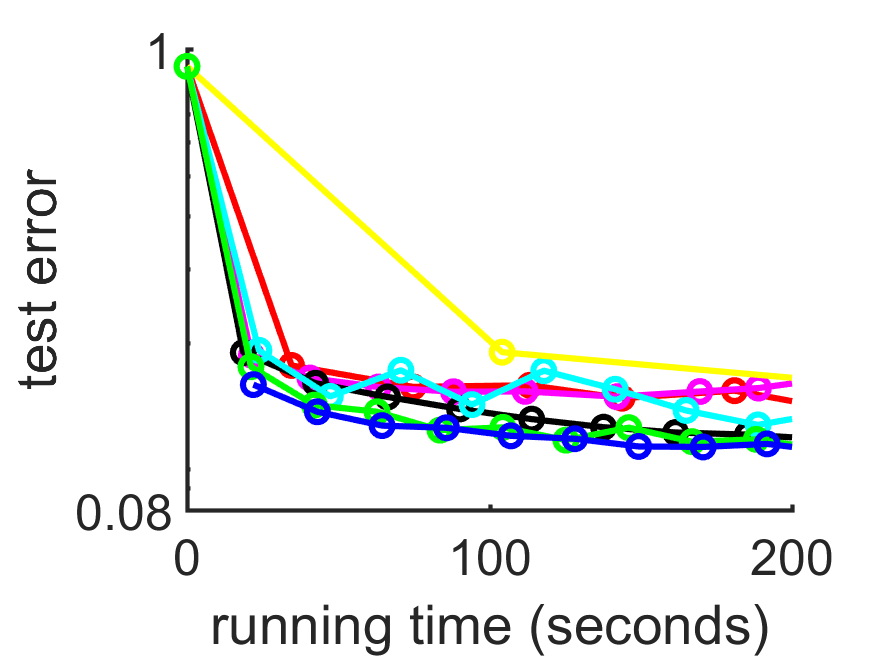}}\\
   \subfigure[p=16]{\includegraphics[width=0.25\linewidth]{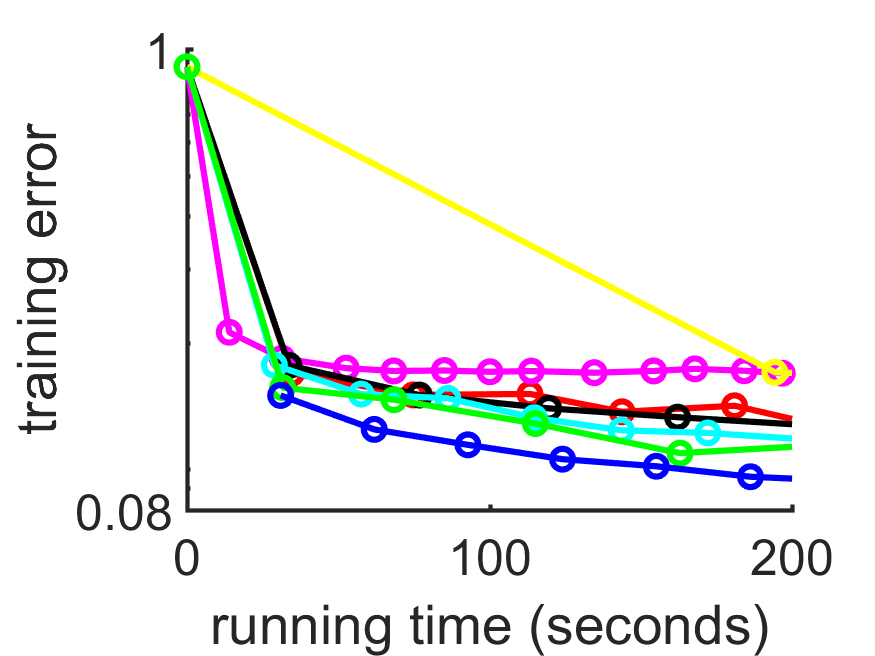}}
  \subfigure[p=16]{\includegraphics[width=0.25\linewidth]{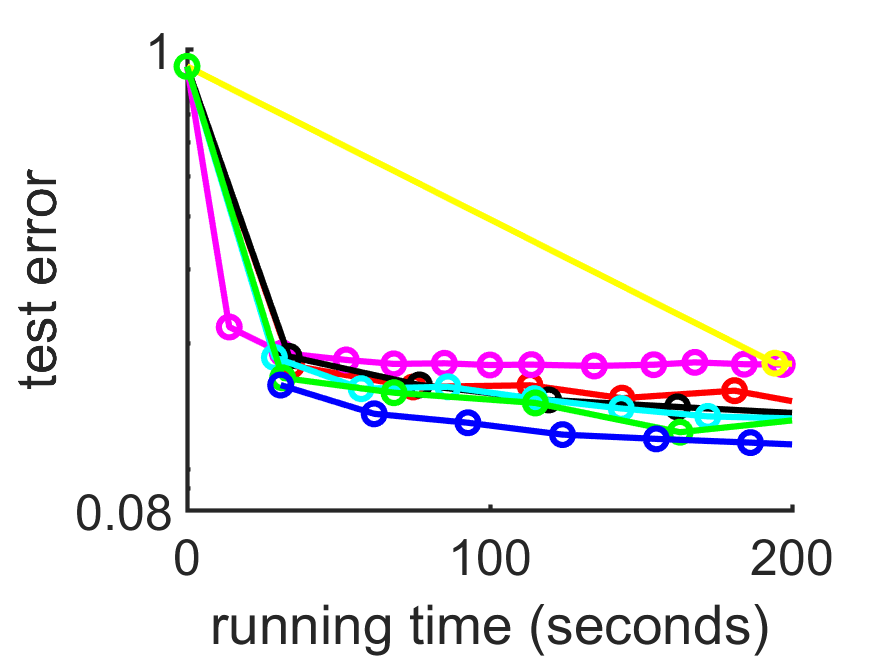}}
  \caption{Experiments on Fashion-MNIST.}\label{fig:fmnistmainex}
\end{figure}

\autoref{fig:cifar10mainex} and \autoref{fig:cifar100mainex} plot the training error, training loss, test error, and test loss of our method and the baselines with different parallel sizes on \textit{CIFAR-10}. 
As can be seen, with the increase of computing processors (workers), the performance of \textit{SPSGD} becomes unstable, due to the fact that averaging the parameters in non-convex cases leads to divergence.
\textit{MMWU} shares the same performance with the sequential SGD which means the sequential SGD is already arriving at an optimal performance of itself.
The worse performance of \textit{OMWU} is caused by the heavy burden of computing the weight based on the whole training samples.
As expected, our novel \textit{WASGD+} method consistently outperforms all the baselines. 

\autoref{fig:fmnistmainex} demonstrates the results on \textit{Fashion-MNIST} and \autoref{fig:mnistmainex} depicts the results on \textit{MNIST}.
Same conclusions are reached.
We can see that our \textit{WASGD+} method consistently outperforms all the benchmarks. 
Although \textit{EASGD} may have a better performance with a careful parameter selection, the algorithm does not offer a standard way of yielding good performance. 
Based on the results, we also find that our method is capable of maintaining a stable performance with large communication period $\tau$.

\begin{figure}[h]
  \centering
  \subfigure[p=4]{\includegraphics[width=0.25\linewidth]{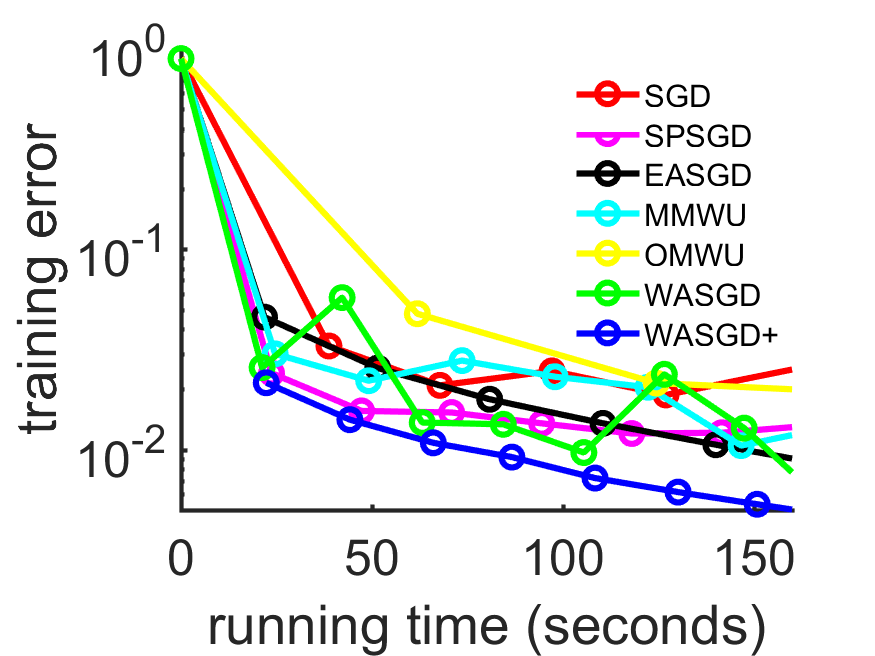}}
  \subfigure[p=4]{\includegraphics[width=0.25\linewidth]{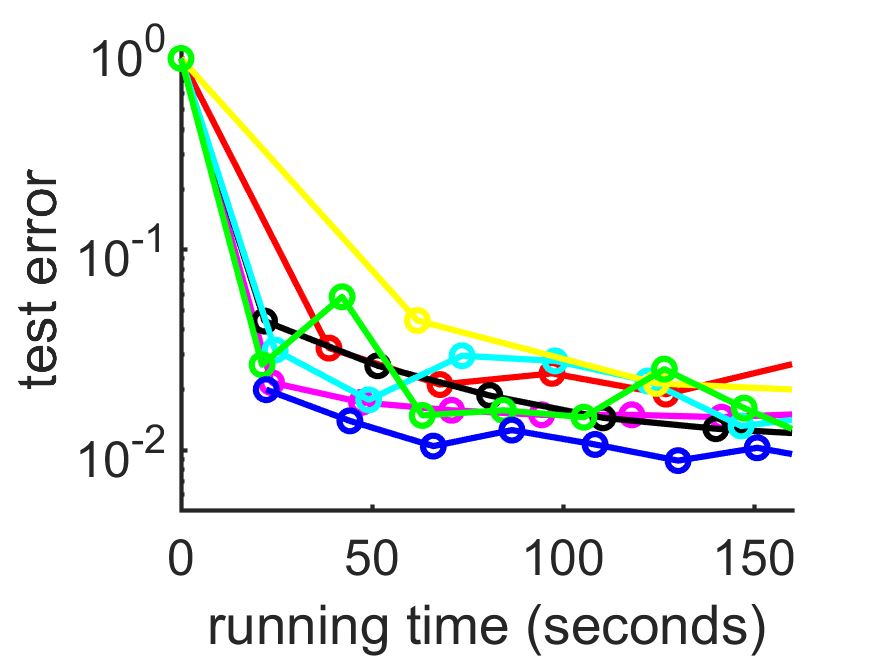}}\\
   \subfigure[p=8]{\includegraphics[width=0.25\linewidth]{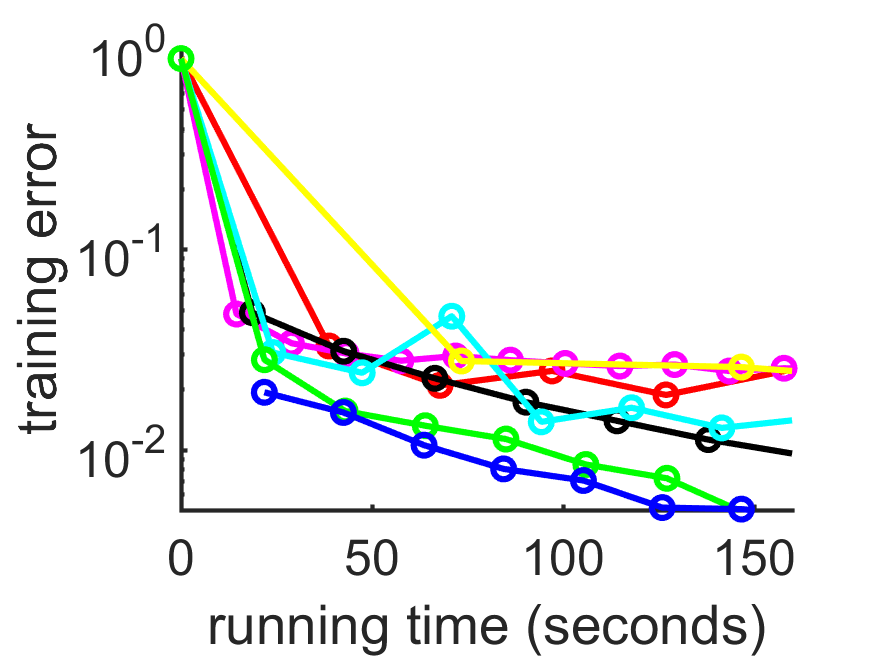}}
   \subfigure[p=8]{\includegraphics[width=0.25\linewidth]{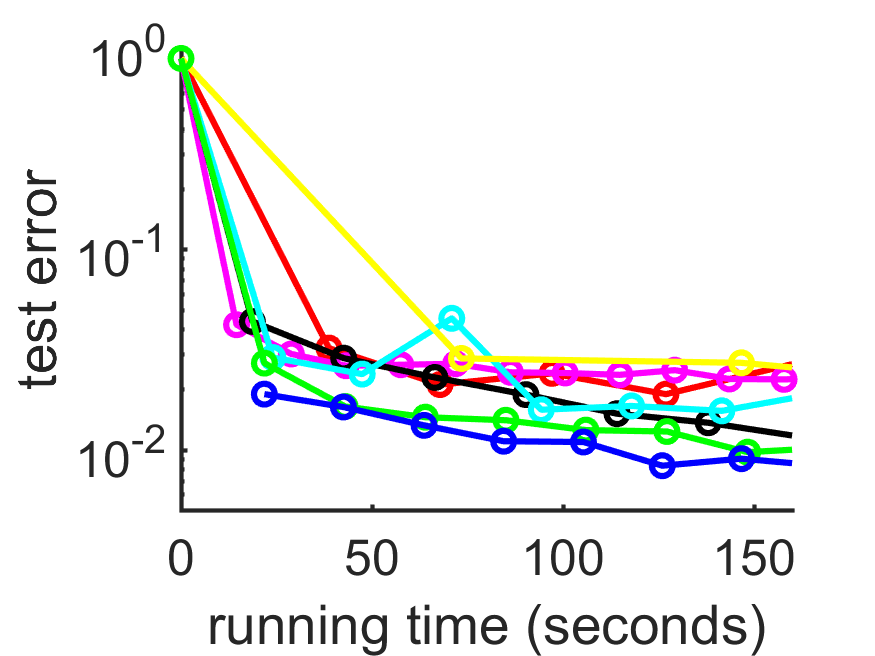}}\\
   \subfigure[p=16]{\includegraphics[width=0.25\linewidth]{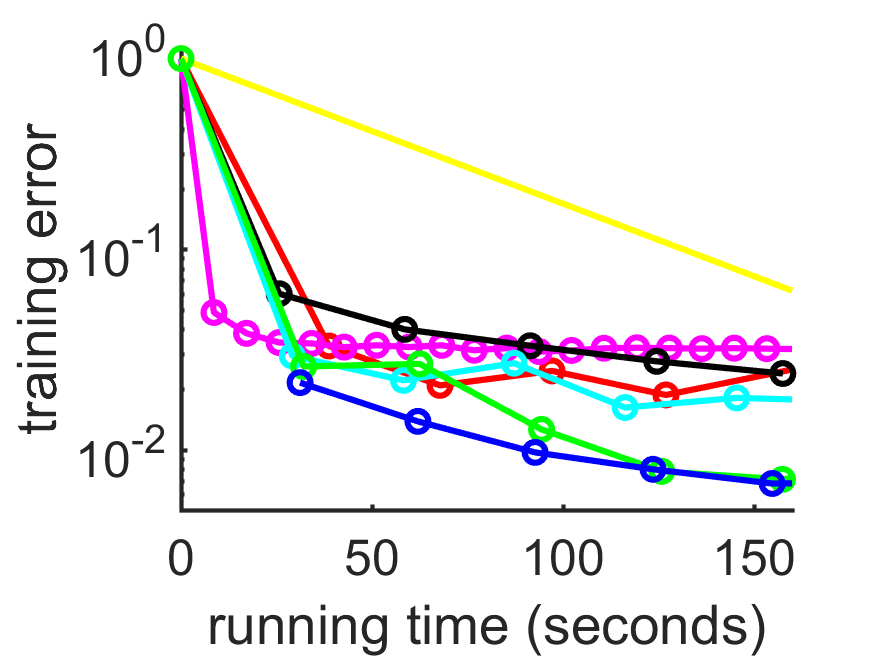}}
  \subfigure[p=16]{\includegraphics[width=0.25\linewidth]{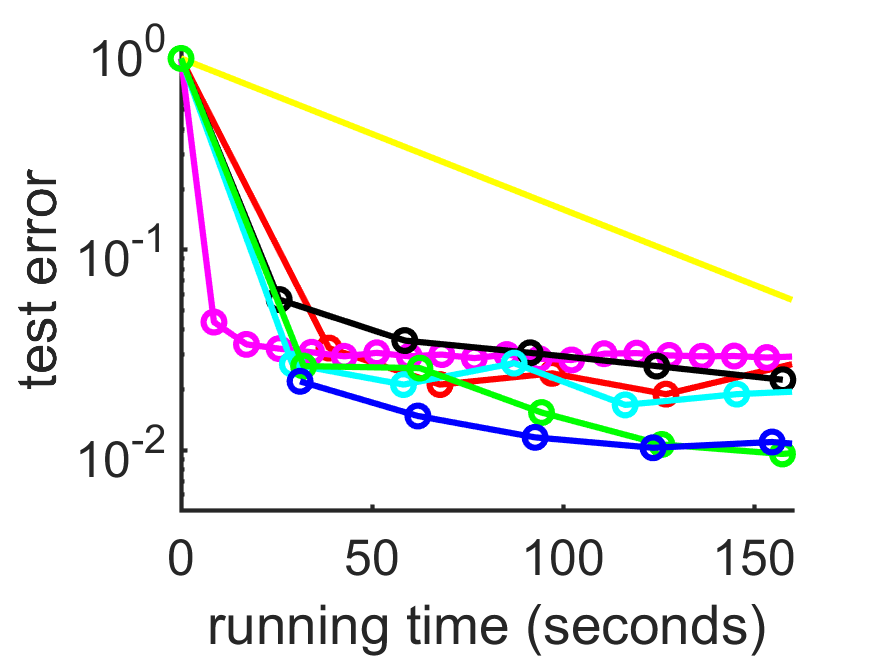}}
  \caption{Experiments on MNIST.}\label{fig:mnistmainex}
\end{figure}

\section{Related Work}
\label{sec:relatedwork}
Multicore  and  distributed  optimization  algorithms  have  aroused  much  attention  in  recent  years \cite{chen2016revisiting}.
Based on the communication method, parallel SGD can be divided into two major types, including (i) centralized algorithms, (ii) decentralized algorithms.
\subsection{Centralized algorithm}
For the implementation of centralized algorithms, the proposed strategy  \cite{zhang2015deep} is that the master will only be responsible for receiving the update parameters.
Once the local worker sends its newest parameters to the master, the master will change part of its current parameters and return which to the local worker.
After receiving the modified result from the master, the local worker also changes part of itself based on information received.
Due to the heavy burden of computing gradients and updating parameters, the case when only the master was in charge of updating the parameters was further discussed in \cite{chen2016revisiting}.
The main task of the local workers  was to calculate the gradients
based on the current parameters and sending them to the master.
In order to avoid the staleness of local workers, they also added extra $b$ backup workers in the training process.
The master should follow the "first come, first serve" principle to update the parameters.
Upon receiving enough number of gradients, the master will update the parameters and send the new parameters to all the local workers.
For those delayed results based on the previous parameters, the master will reject them. 

Following the idea that the first few steps are highly related to
the final result, a new warm up scheme \cite{goyal2017accurate} was adopted in the parallel method.
Such strategy that applies small learning rate at the beginning of training can ensure that the algorithm will approach the local optimal smoothly at the beginning and land at a proper position while converging in the end.
To devote sufficient sources on the searching process, we adopt the decentralized method in our framework.
\subsection{Decentralized algorithm}
Regarding decentralized algorithms, the work in \cite{zinkevich2010parallelized} is considered as the first parallel SGD method.
This method was also the first method attempts to introduce the ideal of weight.
For dealing with large data, it proposed to equally split the data to local workers and let each worker updates itself based on the assigned part of data.
Once the local worker finishes its job, it will send its parameters to all the other workers.
When all the local workers finish their jobs, they will average the sum of received parameters and their own parameters.
Such process can be seen as a special case of equally weighted case by setting the communication period equal to the size of assigned data which is also contained in our method.
Our method can be considered as a more generalized form of \cite{zinkevich2010parallelized}.

In order to update the parameters as soon as possible, the strategy of   \cite{recht2011hogwild} suggested to open the access of global parameters to all the local workers.
As long as the local workers get the updated direction of the global parameters, they can directly modify it.
To avoid eliminating other workers' results, gradient sparsity is required.  
A more general version of \cite{recht2011hogwild} was proposed by \cite{pan2016cyclades}.  
It extended Hogwild to a more common case by letting each worker updates the global model vector stored in a shared memory following some topological structures (e.g. bipartite graph).
By applying such topological structure, it can eliminate the effect of overlapping updated result in the dense case.  
\subsection{Other work}
For the implementation environments, some studies are conducted in distributed systems \cite{dean2012large, zhang2015deep, chen2016revisiting, goyal2017accurate}; while some are implemented in a shared memory architecture  \cite{recht2011hogwild, pan2016cyclades, zhang2016hogwild++}.

Here we also summarize several newly proposed parallel SGD methods.
The work in \cite{ming2018distributed} proposed a distributed version of SGD with variance reduction named \textit{DisSVRG}.
In order to accelerate the training process, it uses an acceleration factor to control the adaptive learning rate along with an adaptive sampling strategy.
The paradigm of \cite{alistarh2018convergence} estimated the cost of asynchronous communication by providing a closer upper and lower bounds.
By discovering the balance between the maximum delay and the SGD convergence rate, it enhances the performance of the parallel SGD.
Finally, the work in \cite{teng2018bayesian} applied the Bayesian method to mitigate the straggler effect in synchronous gradient-based optimization and achieved speedup in the parallel SGD training process.
\section{Concluding Remarks}
\label{sec:conclusion}
In this paper, we modified the parallel SGD method from our previous work \cite{guo2019weighted}, and coined this novel algorithm  \textit{WASGD+}.
In the new method, we connect the quality of samples with sample orders, providing a reasonable way of finding the relatively optimal sample orders based on the performance of all local workers. 
We also provide a more flexible weight evaluating function that can satisfy different requirements of weighing strategies (e.g. equally weighted or bcast the best one \cite{zhang2019parallel}).
The newly proposed weight evaluating function also builds a connection between two independent stochastic processes: SA and SGD.
Since SA has already been thoroughly studied, we can thereby explore the parallel SGD more efficiently.
After proving the convergence and stability of our method, we apply it to CNN, and compare to the state-of-the-art techniques in handling parallel stochastic optimization problems.
Our results suggest a consistent
superiority of our method over other existing methods under various
datasets and with different numbers of computing cores.
Taken together, these affirm the effectiveness and efficiency of our new \textit{WASGD+} method in training deep learning models.
 \bibliographystyle{ieeetr} 
\bibliography{ijcai19.bib} 

\newpage
\setcounter{page}{1}
\appendix
\section{Proof of Property 2}
\setcounter{property}{1}
\begin{property}
If we set equally weighted case as baseline and try different values of $\tilde{a}$. 
As $\tilde{a} \rightarrow \infty$, the performance of weighted case must be worse than the baseline,  while as $\tilde{a} \rightarrow 0$, the weighted case will approach the performance of the baseline. 
\end{property}
\begin{proof}
When $\tilde{a} \rightarrow \infty$, based on Property \ref{property1}, we know that the difference between the weight of smallest energy and others will be increased. 
Such strategy gradually ignores the contribution from other workers and can be seen as a specific case of sequential one.
Relied on the result from \cite{zinkevich2010parallelized}, equally weighted case performs better than sequential algorithm.

As $\tilde{a} \rightarrow 0$, the difference between the weight of each local worker will be decreased.
As the difference of weight among each worker become smaller and smaller, it will gradually transfer to the equally weighted case.
\end{proof}

\section{Additional Algorithms}
Here we show algorithms of the original \textit{WASGD} and the asynchronous version of \textit{WASGD+}.
\subsection{Main Algorithm for WASGD}
Algorithm \ref{alg:4} was used to calculate results (e.g. loss, accuracy) in \textit{WASGD} on the dataset.
The weight evaluating function was set to be $\frac{1}{h}$.
\begin{algorithm}[h]
\SetAlgoLined
\SetKwInOut{Input}{Input}
 \Input{learning rate $\eta$, decision choice $\beta$, estimation number $m$, communication period $\tau \in N$, M samples in datasets $D$, initial parameters $x^i=x$, $p$ local workers.}
 
 \While{The stopping criteria is not met}{
 \For{$l$ = 0,1,...,$\frac{M}{\tau}$}{
 $h^i\leftarrow 0$\;
 \For {k =$1,...,\tau$}{
  $x\leftarrow x^i$\;
 \If{ $k \geq \tau-m$}{
        $h^i \leftarrow h^i+loss(x^i,D\left[l\frac{M}{\tau}+A[k]\right]);$\\
       
        }
  
  \If{ \textit{k} divides $\tau$ }
  {
    Send $(h^i,x_k^i,i)$ to $p-1$ workers\;
    Wait for$(h^j,x_k^j,j)$ from $p-1$ workers\;
    Arrange $h$ and $x$ based on received index\;
    $\theta^i \leftarrow \frac{{1}/{h^i}}{\sum_{j=1}^p {1}/{h^j}}$\;
    $x^i \leftarrow {\sum_{j=1}^{p}\theta^j x^{j}}$\;
    $h^i \leftarrow 0$ \;
  }
    
    $x^i \leftarrow x^i-\eta g^i_{ite}(x,D\left[l\frac{M}{\tau}+A[k]\right])$\;
      }
 }
 }
 \caption{Synchronous WASGD: Processing by worker \textit{i}}
 \label{alg:4}
\end{algorithm}

\subsection{Asynchronous WASGD+}
In terms of asynchronous method, We add additional $b$ backup workers.
Each local worker is only responsible for updating its own parameters and go through the samples in different orders. 
The iteration will increase by one after each update.
when the iterations are the multiples of communication period $\tau$, as long as the local worker received $p-1$ results, it will begin to update the result without waiting other b local workers.
Upon it receiving an aggregation result based on the results from $p-1$ workers, the search will continue.
During communication, the estimated loss will be used to generate score and weight for the current parameters' performance. 
Algorithm \ref{alg:3} provides the main process of asynchronous version about our method.
\begin{algorithm}[h]
\SetAlgoLined
\SetKwInOut{Input}{Input}
 \Input{learning rate $\eta$, decision choice $\beta$, estimation number $m$, divided order part number $n$, divided communication part number $c$, communication period $\tau \in N$, M samples in datasets $D$, initial parameters $x^i=x$, $p+b$ local workers.}
 
 B $\leftarrow$ RecordIndex($D, m,c,\tau$)\;
 Scores, Seed $\leftarrow$ zeros($n, 1$)\;
 \While{The stopping criteria is not met}{
 \For{$l$ = 0,1,...,$n-1$}{
 $h^i,score \leftarrow 0$\;
 $A,$ Seed[$l$]=OrderGen(Scores[$l$], Seed[$l$], $\frac{M}{n}$);\\
 \For {k =$1,...,\frac{M}{n}$}{
  $x\leftarrow x^i$\;
 \If{ $k \in B$}{
        $h^i \leftarrow h^i+loss(x^i,D\left[l\frac{M}{n}+A[k]\right]);$\\
       
        }
  
  \If{ \textit{k} divides $\tau$ }
  {
    Send $(h^i,x_k^i,i)$ to $p+b-1$ workers\;
    Wait for$(h^j,x_k^j,j)$ from $p+b-1$ workers\;
    Arrange $h$ and $x$ based on received order\;
    \If{ worker $i$ receives $p-1$ results}
    {
    $\theta^i \leftarrow \frac{e^{-\tilde{a}h^i/\sum_{j=1}^p h^j}}{\sum_{g=1}^p e^{-\tilde{a}h^g/\sum_{j=1}^p h^j}}$\;
    $x^i \leftarrow x^i(1-\beta)+\beta{\sum_{j=1}^{p}\theta^j x^{j}}$\;
    $score+\leftarrow$ Judge($h^1,...h^p,i,p$)\;
    $h^i \leftarrow 0$ \;
    }
  }
    
    $x^i \leftarrow x^i-\eta g^i_{ite}(x,D\left[l\frac{M}{n}+A[k]\right])$\;
      }
   Scores[$l$]=$score$
 }
 }
 \caption{Asynchronous WASGD+: Processing by worker \textit{i}}
 \label{alg:3}
\end{algorithm}
\end{document}